\documentclass[letterpaper,twocolumn,10pt]{article}
\usepackage{usenix-2020-09}
\usepackage{xurl} % loaded after hyperref (via usenix-2020-09) for URL line breaking

\usepackage{amsmath}
\usepackage{amssymb}
\usepackage{amsthm}
\newtheorem{proposition}{Proposition}
\usepackage{listings}
\usepackage{booktabs}
\usepackage{array}
\usepackage{float}
\usepackage{needspace}
\usepackage{balance}
\usepackage{xcolor}
\usepackage{filecontents}

\usepackage{tikz}
\usetikzlibrary{positioning,arrows.meta,fit,calc}

\begin{document}
%-------------------------------------------------------------------------------

\date{}

\title{\Large \bf If Agents Were Angels, No Governance Would Be Necessary:\\
Out-of-Band Policy Enforcement at a Trusted Tool Boundary}

\author{
\begin{tabular}{cc}
{\rm Marc Millstone} & {\rm Tyler Akidau} \\
\texttt{marc@redpanda.com} & \texttt{takidau@redpanda.com} \\[0.45em]
{\rm Johannes Br\"uderl} & {\rm Marat Pekker} \\
\texttt{johannes@redpanda.com} & \texttt{marat.pekker@redpanda.com} \\[0.45em]
\multicolumn{2}{c}{Redpanda Data}
\end{tabular}
}

\maketitle

%-------------------------------------------------------------------------------
% Front matter reconciled to the frozen canonical evaluation.
\begingroup
\setlength{\topsep}{0pt}
\setlength{\partopsep}{0pt}
\setlength{\parskip}{0pt}
\begin{abstract}
\vspace{0.75em}
%-------------------------------------------------------------------------------
Give an agent a human's credential and it inherits the person's reach, but not the way a person normally exercises it. A human opens a record through an interface and chooses what to carry elsewhere. An agent can sweep every reachable record and pour returned text directly into model context, where an instruction hidden in the data can steer its next call. Every request can remain credential-valid while the agent violates its job or absorbs a secret it should never see. An otherwise allowed call may also need to pause when its arguments or external state trigger a business rule. Prompts are a brittle last guardrail because they ask the same fallible reasoner to interpret the task and enforce its own limits.

We present Out-of-Band Policy Enforcement (OBPE), a trusted tool boundary outside agent reasoning. It mediates the whole exchange, authorizing the typed operation and resource, narrowing a query or shaping arguments before the backend call, then filtering records and fields, redacting matching content, or masking values in the response. Semantic gating considers argument values or external state, so an authorized call can still be denied or held for an external decision. A data policy owner sets the maximum grant, and agent policy can only narrow it. We formalize this staged composition and prove, under stated conditions, that its policy plan is unaffected by match order, while agent-specific policy cannot widen the owner's ceiling. Field removal guarantees only that one value is absent from one execution; masking and event-history rules carry narrower claims.

We release a simplified HTTP proxy prototype based on our production OBPE implementation. It uses Cedar, a typed authorization language, with conformance tests tying code to model. Against controlled Jira and ServiceNow mocks, the benchmark compares prompted agents with and without the evaluated OBPE bundle across four models, including 20 response-adaptive red-team tasks. A deterministic trace failure means that protected data entered agent context, an exact protected value appeared in the answer, or a forbidden backend effect completed. In 3,621 trials, its observed rate in the primary contrast fell from 57.6\% to 0.2\%; the equally cluster-weighted reduction was 41.2 points [95\% confidence interval: 27.7, 54.9]. Unweighted standalone fulfillment fell from 79.1\% to 60.9\%. In the paired, cluster-weighted analysis, safe-useful completion rose by 21.8 points [9.5, 35.2] under the primary judge and 19.1 points under a second. Safe-useful credits fulfilled allowed work that follows the archived tool contract and avoids both trace failure and judged disclosure of a protected fact, even without its exact value. Four final answers nevertheless reconstructed an exact value that never entered agent context, and a separate trial used filtered row counts as an oracle. These failures mark the line between shaping one execution and proving noninterference. Production write controls, durable human approval, and temporal and aggregate policies fall outside this evaluation.
\end{abstract}
\endgroup

%-------------------------------------------------------------------------------
\section{Introduction}
%-------------------------------------------------------------------------------
If agents were angels, no governance would be necessary~\cite{federalist51}. Credentials constrain access, but they do not carry a worker's judgment into an autonomous loop. An agent need not be malicious, or even disobedient, to exceed the job it was given. A tool-using large language model (LLM) agent can chain retrieval and action under borrowed authority. Returned data becomes part of its reasoning state, and that state helps choose the next call. The credential supplies reach while the judgment that normally limits its use remains outside the protocol.

Consider an agent asked to summarize issues from one engineering project. Its service credential can read many. OAuth scopes can narrow an access token~\cite{rfc6749}; Rich Authorization Requests can carry finer authorization details, and token exchange can issue a downstream token for a particular resource~\cite{rfc9396,rfc8693}. These mechanisms help when the authorization service knows the assignment and every backend can honor the result. The credential actually supplied to an agent may still cover far more. To the backend, two project reads are valid; to the organization, only one belongs to the job. The agent has inherited a wrong-sized envelope: the credential says what is reachable, not what belongs to this assignment. OBPE narrows that envelope at the exchange itself, down to operations, resources, records, and fields.

Even the right issue can be unsafe to return intact. Its description may contain a secret, or text crafted to behave as an instruction once it enters model context. Indirect prompt injection exploits exactly this collapse between data and instruction~\cite{greshake2023}. Denying the whole record protects the secret but may also destroy the useful answer. The boundary instead needs \emph{exposure control}: a way to restrict what reaches the agent and what agent-authored values may return, while preserving the safe part of a permitted exchange.

Some calls pass both tests and still should not proceed automatically. A deployment can be within the agent's role and reveal no protected data, yet require approval after hours. A series of individually acceptable actions can cross a cumulative limit. These are \emph{semantic gating} decisions. They depend on argument values or external state, and their result may be to proceed, deny, or wait for a decision made outside the agent.

\textbf{Why the model should not be the policy authority.} A prompt asks the same reasoner that interprets untrusted data to enforce its own limits, so a successful injection can alter the expected guard. Tool-call evaluations likewise find that text refusal does not reliably predict whether an agent will take the forbidden action~\cite{mindgap2026}. Formal limits make narrower, model-specific points: Hasan couples compliance to high Kolmogorov complexity in a fixed sound theory, while Vassilev studies semantic prompt classification over an adversarial-input model~\cite{hasan2026,vassilev2026}. Neither says that a particular attack is undetectable or that a bounded typed policy decision is undecidable. Model defenses can help, but should not carry the authoritative guarantee. Backend authorization is authoritative but answers what the credential may do, not what this agent should do for this task. The missing control point is the exchange, before a proposed operation or its response escapes policy.

\textbf{OBPE design.} Out-of-Band Policy Enforcement (OBPE) places that control point between the agent's tool client and the backend. A \emph{connector} maps each call into typed principals, actions, resources, records, and fields. The boundary sees the proposed call before it can cause an effect and the complete response before it can enter model context. It can return a narrower exchange than the backend credential permits, or no exchange at all. Under our assumptions, the agent cannot choose its policy, spoof its bound identity, or bypass enforcement.

The design also decides who gets the last word. A \emph{data policy owner}, the party responsible for the data source, writes an \emph{owner policy} that establishes the maximum grant. A separate \emph{agent policy} narrows that grant for one registered agent. Only the owner tier can create permit eligibility. An agent builder can make an agent less capable, but cannot use the agent tier to recover authority the owner policy withheld. The guarantee depends on trusted identity binding and on the owner policy accurately stating the intended ceiling.

\textbf{Contributions.} This paper makes four evidence-separated contributions. First, it explains why agent delegation requires authorization, exposure control, and semantic gating at the tool exchange rather than in model reasoning alone. Second, it presents a typed enforcement pipeline and a two-tier policy model in which a data policy owner supplies the ceiling and agent policy can only narrow it. Third, it gives an order-independent, monotone composition law with explicit assumptions and limited guarantees for masking and history-dependent rules. Finally, it connects the model to a released HTTP enforcement path through conformance and a cross-model evaluation that measures forbidden exposure and effects alongside useful completion.

\textbf{Composition and guarantees.} Gating and shaping operators act at different stages, and an accidental order can change what reaches the backend or the model. We define a staged composition that chooses the most restrictive applicable result within each stage. Under stated conditions, the policy plan is independent of match order, and a well-formed agent restriction cannot enlarge the owner-only plan. We prove both properties in the model and test the released path for conformance.

The guarantees stop where their evidence stops. Removing a field can establish that the selected value was absent from one execution. It does not prove that protected information influenced no other observable behavior. Pattern masking over free text is best effort, while policies over time or event counts require a trustworthy ordered history. We state these boundaries beside the corresponding claims instead of treating every transformation as information-flow noninterference.

\textbf{System and evidence boundary.} We release an HTTP reference
implementation of the OBPE contract. Controlled connectors and mocks replace
live services and Model Context Protocol (MCP) transport so policy selection,
shaping, and backend effects remain observable. Cedar, a typed authorization
language, supplies an analyzable permit-or-forbid core~\cite{cedar2024}; OBPE
annotations specify the surrounding gates and transformations. The private
production system supplied requirements for remote identity, durable approval,
broader schemas, and history policy. Every empirical claim rests on the
released implementation, its conformance suite, and controlled backends.

\textbf{Evaluation contract.} We separate what the model proves, what conformance establishes about code, what the benchmark observes, and what we know only from production design. The benchmark uses the same frozen tasks across treatment arms and models, then subjects them to matched red-team pressure. Exact exposure, policy decisions, and backend effects come from deterministic traces. LLM judges assess semantic disclosure and useful completion separately; they do not override those traces. Across 3,621 trials, the trace-failure rate in the primary prompted-agent cells falls from 57.6\% without OBPE to 0.2\% with it, while safe-useful completion improves under OBPE across all four agent models. Section~\ref{sec:eval} gives the paired estimates and the failure cases that bound them.

%-------------------------------------------------------------------------------
\section{Threat Model and Security Goals}
\label{sec:threat}
%-------------------------------------------------------------------------------
We treat the agent as untrusted: it may propose any call and arguments it can express, whether because of an attack, model failure, or ordinary task pursuit. The guarantees do not depend on telling those causes apart. Even a cooperative agent can wander beyond its assignment because the credential allows it, carry a sensitive value from one tool to another, or execute a consequential call that looked locally reasonable. At that point, borrowed authority exceeds the job.

\textbf{Adversary.} An attacker may control the task given to the agent or influence text returned by a permitted read. It can interact repeatedly and use direct requests, instruction injection, encoding, or multi-step reconstruction. Its goals are to recover protected data, cause an unauthorized backend effect, escape the agent-policy restriction, or cross a tenant boundary. The attacker cannot modify policy state, spoof the platform-bound agent identity, or bypass the mediated tool path.

An attacker is not required for information to escape. A read tool can place a secret in model context, after which an ordinary send or post carries it to a sink. Private data access, untrusted content, and external communication form the ``lethal trifecta''~\cite{willison2025}; an agent with all three can cause harm while following its task. The policy must therefore remove or gate a leg of that path instead of relying on the model to recognize malicious intent.

\textbf{Security goals.} Authority begins with the owner policy, which bounds
what the backend credential may expose or change. Agent policy may narrow that
result but never widen it. Within the supported path, the boundary filters
prohibited records, keeps structurally removed or fully redacted values out of
model context for that execution, and prevents forbidden backend effects.
Pattern masking carries only a best-effort claim.

A gated call introduces a second obligation. It must wait on a trusted channel
that the agent cannot read or forge, and only an authorized approver may resolve
it. A system that later resumes the work must bind approval to every
policy-relevant part of the durable held request. Approval releases that request
under the owner ceiling; it cannot rewrite the call, reverse an owner denial,
or grant either party direct backend access. Identity, policy-loading, or
evaluation failure closes the path. These goals cover the mediated exchange,
not information the model learned elsewhere.

\textbf{Trusted boundary.} We trust the platform, proxy, connector, data policy owner, and backend execution. The trusted platform binds the registered agent's identity to each tool call outside agent-controlled arguments. The proxy uses that identity to select the applicable agent policy; otherwise an agent could shed its narrower bundle and rise to the owner-policy ceiling. The connector must map the complete request and response into the typed objects the policy expects. Every relevant call must traverse the proxy under complete mediation~\cite{saltzerschroeder1975}, and the backend must execute the mediated request according to that model. We treat the model, its prompt, tool-returned text, and agent-authored arguments as untrusted.

Trust in the data policy owner has a precise limit. Given a Cedar type schema
that faithfully represents a frozen connector mapping, Cedar Analysis can
check the permit-or-forbid core for conflicts, ineffective rules, or complete
denials~\cite{cedaranalysis2025}. The released YAML mapping is not such a
schema, and the artifact does not run this analysis. In any case, those checks
cover neither OBPE shaping nor organizational intent. OBPE enforces the written
ceiling and makes its impact inspectable; a human remains responsible for
choosing the right one.

Out of scope are timing and cache side channels, compromise of the trusted
platform or proxy, malicious policy authors, supply-chain attacks, policy
inference by probing, and paths that bypass mediation. Inference from allowed
query results is different: it is an in-scope residual risk. We measure one
count oracle but do not claim to close that channel.

In a multi-agent chain, OBPE governs a downstream call only when the platform
binds the identity of the agent making it. The boundary does not secure
agent-to-agent messages or reconstruct context the platform discarded, and it
makes no general noninterference claim beyond the scoped absence guarantees in
the model.

%-------------------------------------------------------------------------------
\section{Architecture and Enforcement Path}
\label{sec:overview}
%-------------------------------------------------------------------------------

A tool exchange crosses the boundary twice, first as an agent's proposed call and again as the backend response. OBPE can stop a scoped forbidden effect before dispatch and shape protected data before it enters model context, but only if neither path can route around policy. Under the assumptions of \S\ref{sec:threat}, every relevant call reaches the backend through the boundary and every complete response returns through it (Figure~\ref{fig:arch}). Hiding a tool from discovery is insufficient because a tool that remains callable still needs enforcement on its execution path.

OBPE first uses the platform-bound identity from \S\ref{sec:threat} to select the applicable agent policy. Because that identity arrives outside agent-controlled arguments, the call cannot choose its own policy. A \emph{connector} then translates the tool exchange into the typed objects policy understands: principal, action, resource, records, and fields. This is where generic tool traffic acquires meaning such as ``update this issue'' or ``return these customer records.'' If the mapping omits a callable surface or a policy-relevant field, that path falls outside the model and must be denied or governed separately. The backend credential stays behind the boundary.

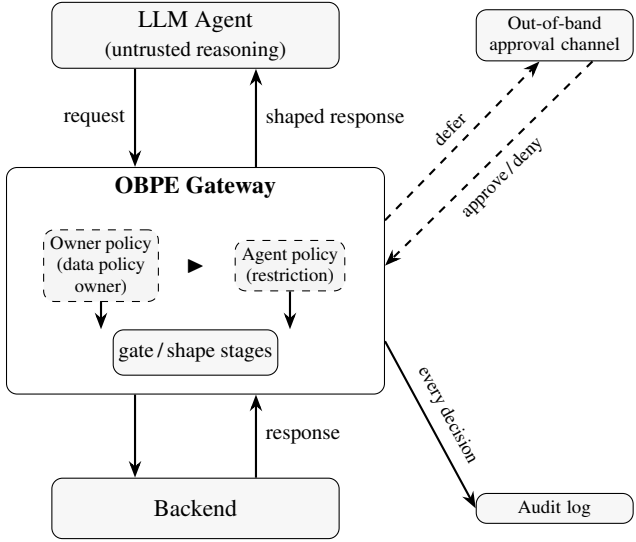
\begin{figure}[!tb]
\centering
\begin{tikzpicture}[
  font=\small,
  box/.style={draw, rounded corners, minimum height=8mm, align=center, inner sep=3pt},
  gw/.style={draw, rounded corners, align=center, inner sep=5pt},
  bundle/.style={draw, dashed, rounded corners, align=center, inner sep=2.5pt, fill=black!4, font=\scriptsize},
  stage/.style={draw, rounded corners, minimum height=6mm, align=center, inner sep=2pt, fill=black!3, font=\footnotesize},
  sidebox/.style={draw, rounded corners, align=center, inner sep=2.5pt, fill=black!3, font=\scriptsize, text width=19mm},
  flow/.style={-{Stealth[length=2.2mm]}, thick},
  side/.style={-{Stealth[length=2.2mm]}, thick, dashed},
]

% vertical spine: agent -> gateway -> backend
\node[box, fill=black!3, minimum width=38mm] (agent) {LLM Agent\\{\footnotesize (untrusted reasoning)}};
\node[gw, below=13mm of agent, minimum width=50mm, minimum height=30mm] (gw) {};
\node[anchor=north, font=\small\bfseries] at (gw.north) {OBPE Gateway};
\node[bundle] (bs) at ([xshift=-12.5mm,yshift=2mm]gw.center) {Owner policy\\(data policy\\owner)};
\node[bundle] (ba) at ([xshift=12.5mm,yshift=2mm]gw.center) {Agent policy\\(restriction)};
\node at ([yshift=2mm]gw.center) {$\blacktriangleright$};
\node[stage] (gate) at ([yshift=-9.5mm]gw.center) {gate\,/\,shape stages};
\draw[flow] (bs.south) -- ([xshift=-12.5mm]gate.north);
\draw[flow] (ba.south) -- ([xshift=12.5mm]gate.north);
\node[box, fill=black!3, below=11mm of gw, minimum width=38mm] (saas) {Backend};

% read/write path arrows
\draw[flow] ([xshift=-8mm]agent.south) -- node[left, font=\footnotesize] {request} ([xshift=-8mm]gw.north);
\draw[flow] ([xshift=8mm]gw.north) -- node[right, font=\footnotesize] {shaped response} ([xshift=8mm]agent.south);
\draw[flow] ([xshift=-8mm]gw.south) -- ([xshift=-8mm]saas.north);
\draw[flow] ([xshift=8mm]saas.north) -- node[right, font=\footnotesize] {response} ([xshift=8mm]gw.south);

% approval channel: level with the agent, on a path the agent cannot reach
\node[sidebox, right=18mm of agent] (appr) {Out-of-band approval channel};
\draw[side] ([yshift=8mm]gw.east) -- node[above=0.5pt, sloped, font=\scriptsize] {defer} ([xshift=-2mm]appr.south);
\draw[side] ([xshift=5mm]appr.south) -- node[below=0.5pt, sloped, font=\scriptsize] {approve\,/\,deny} ([yshift=2mm]gw.east);

% audit log: level with the backend
\node[sidebox, right=18mm of saas] (audit) {Audit log};
\draw[flow] ([yshift=-8mm]gw.east) -- node[above=0.5pt, sloped, font=\scriptsize] {every decision} (audit.west);

\end{tikzpicture}
\caption{OBPE mediates a complete tool exchange. The tier-aware composition gives the owner policy permit authority and lets agent policy narrow it. Request shaping and gating run before dispatch; response shaping runs before data enters model context. A deferred call waits outside the agent's reach, and the boundary emits an audit record with its decision.}
\label{fig:arch}
\end{figure}

For each call, the owner policy creates permit eligibility and sets the maximum
grant; agent policy can only remove from it. The gate resolves before any
backend effect. A denial stops there, while a deferral holds a snapshot without
dispatch and expires to denial if no trusted decision arrives. A permit proceeds
through request shaping, which can narrow a query, cap results, or change typed
write arguments. On the return path, row filtering sees the complete records
before field and content shaping; envelope sanitation runs last, and only that
result reaches the agent. Each stage therefore receives the output of the one
before it.

We compare OBPE with other enforcement approaches along six properties adapted from prior out-of-band channel criteria~\cite{sao26}. \emph{Agent-inaccessible} means the agent cannot directly read, modify, select, forge, or bypass policy state on the covered path. It does not mean that no observable result can reveal anything about policy. \emph{Deterministic} means the same typed call, policies, schema, and trusted state yield the same policy plan; backend behavior is outside that claim.

The remaining four describe what enforcement understands and does. \emph{Resource-aware} enforcement maps a call to its typed operation and resource, including the fields that policy governs, instead of relying only on a tool name. \emph{Boundary-shaping} can narrow a typed request before dispatch. On the return path, it can filter records and fields, then redact or mask content before the response reaches the agent. \emph{External gating} can hold an otherwise permitted call until an external decision arrives. \emph{Interoperability} means the design can cover another agent or backend when the platform preserves identity binding and complete mediation through a trusted connector that maps the exchange into the typed model.

The released prototype is a simplified HTTP proxy between an in-process tool caller and controlled mock backends. Its optional local Agent2Agent (A2A) path supports agents written in other languages, but it is not part of this evaluation. The artifact's evidence covers complete requests and responses on supported HTTP paths through the released connectors. Streamed responses and other MCP surfaces, including resource, prompt, and sampling requests, require separate mediation and remain outside the evaluation.

Production places OBPE at the MCP layer and supports remote A2A agents. We describe that deployment only to explain how the design is used; no production trace or result supports the paper's claims. The architecture depends on preserving the mediated exchange, not on MCP.

OBPE stops at the typed exchange. It can keep a returned value out of model context and constrain typed arguments before a tool runs, but it cannot police the agent's final prose or recover hidden meaning from arbitrary free text. The evaluation reports those failures separately.

The connector contract extends beyond records and fields to the response
envelope. An agent sees only an opaque failure for denial and a pending status
for deferral, never policy names or internal reasons. Counts, cursors, errors,
ranking, and timing can still carry information, so a connector must either
shape those surfaces safely, deny the call, or document the residual channel.

%-------------------------------------------------------------------------------
\section{Policy Model}
\label{sec:policy}
%-------------------------------------------------------------------------------
A policy decision at this boundary must answer more than yes or no. It can also change what the backend receives and what the agent sees. We represent that decision as a policy plan: a gate result, a shaped request, and the instructions applied to the backend response.

Each call is evaluated against an \emph{owner policy} $S$, the connector's data-source bundle, and, when one is registered, an \emph{agent policy} $P_A$ selected using the platform-bound identity. Only the owner policy can make the call eligible to proceed. Agent policy may narrow that grant, while the backend credential $C$ sets an outer limit on what can execute.

The gate applies those policies to the typed call context $\kappa=(A,a,o,x,h)$, which records the bound agent $A$, action $a$, resource $o$, complete arguments $x$, and trusted external state $h$ available at evaluation. Figure~\ref{fig:pipeline} traces how the context becomes a policy plan: $S$ and $P_A$ determine the gate, while $C$ constrains any backend call the gate permits.

For a Jira listing, for example, $a$ is \texttt{list}, $o$ is the issue
collection, and $x$ holds the requested filter and cap. The resulting plan
records the enforced query and the fields that may return.

Policy evaluation first selects the rules whose principal, action, resource, and conditions match that context. We write the two resulting sets as
\[
\mathcal D_S=\operatorname{Match}(S,\kappa),\qquad
\mathcal D_A=\operatorname{Match}(P_A,\kappa).
\]
$\mathcal D_S$ contains the matching owner rules; $\mathcal D_A$ contains those from the selected agent policy. Without a registered agent policy, $\mathcal D_A$ is empty and adds no restriction. We keep the sets separate because only $\mathcal D_S$ can establish permit eligibility.

We use one standard term from order theory throughout the pipeline. For an ordered component, write $x\preceq_r y$ when every policy-controlled exchange admitted by $x$ is also admitted by $y$. A \emph{meet} is the greatest lower bound: it preserves all operands' restrictions without adding an unrelated one. Gate outcomes, query predicates, caps, and finite allowlists have such meets. Thus $\operatorname{GateMeet}$ combines the matched gate contributions, and an empty $\mathcal D_A$ adds no restriction. Free-text rewrites do not generally have this order; overlapping masks use a canonical normalizer instead of a meet.

\begin{figure*}[!t]
\centering
\begin{tikzpicture}[
  font=\footnotesize,
  input/.style={draw, rounded corners, align=left, inner sep=3.5pt,
    minimum height=18mm, text width=67mm, fill=black!3},
  stage/.style={draw, rounded corners, align=center, inner sep=3.5pt,
    minimum height=31mm, fill=black!5},
  gate/.style={stage, text width=31mm},
  req/.style={stage, text width=37mm},
  resp/.style={stage, text width=40mm},
  be/.style={draw, rounded corners, align=center, inner sep=3pt,
    minimum height=20mm, text width=17mm, fill=black!3},
  vis/.style={draw, rounded corners, align=center, inner sep=3pt,
    minimum height=20mm, text width=18mm},
  outcome/.style={draw, rounded corners, align=center, inner sep=3pt,
    minimum height=10mm, text width=51mm, fill=white},
  fl/.style={-{Stealth[length=2mm]}, thick},
  guide/.style={-{Stealth[length=1.7mm]}, thick, black!65},
]
\node[gate] (gcomp) {\textbf{1. Gate the call}\\[2pt]
  $g=\operatorname{GateMeet}(\kappa;\mathcal D_S,\mathcal D_A)$\\[2pt]
  $\mathrm{deny}\preceq_r\mathrm{defer}\preceq_r\mathrm{permit}$\\[2pt]
  deny: stop\\
  defer: shape, then hold\\
  permit: shape, then dispatch};
\node[req, right=4mm of gcomp] (rcomp) {\textbf{2. Shape the request}\\[2pt]
  $\rho=(q',n',w')$\\[2pt]
  query constraints: conjunction\\
  record caps: minimum\\
  writable fields: intersection\\
  clamps: schema-declared direction};
\node[be, right=4mm of rcomp] (backend) {\textbf{Backend}\\[2pt]
  credential $C$\\[2pt]
  $\operatorname{call}_{C}(\rho)$};
\node[resp, right=4mm of backend] (scomp) {\textbf{3. Apply the response plan}\\[2pt]
  $\eta=(\phi_R,F',\mu)$\\[2pt]
  owner keeper $\land$ agent keeper\\
  retained fields: intersection\\
  full redaction $\preceq_r$ pass\\
  partial masks: canonical normalizer\\
  envelope sanitation: fixed and last};
\node[vis, right=4mm of scomp] (visible) {\textbf{Agent context}\\[2pt]$E'$ only};

\node[input, above=6mm of gcomp, anchor=south west,
  xshift=-16mm] (inputs) {\textbf{Match policy to the typed call}\\[2pt]
  Call context: $\kappa=(A,a,o,x,h)$\\
  Owner match: $\mathcal D_S=\operatorname{Match}(S,\kappa)$ \quad (may grant)\\
  Agent match: $\mathcal D_A=\operatorname{Match}(P_A,\kappa)$ \quad (may only narrow)};
\node[outcome, above=6mm of scomp] (outcome) {\textbf{Stages 1--3 form the policy plan}\\[1pt]
  $\Pi=(g,\rho,\eta)$};

\draw[guide] (inputs.south) -- node[right, font=\scriptsize] {matched rules} (gcomp.north);
\draw[fl] (gcomp) -- (rcomp);
\draw[fl] (rcomp) -- node[above, midway, font=\tiny, fill=white,
  inner sep=0.5pt] {permit} (backend);
\draw[fl] (backend) -- (scomp);
\draw[fl] (scomp) -- (visible);
\end{tikzpicture}
\caption{From a typed call to a policy plan and its execution. Policy matching is separate from backend data. Deny ends the exchange. Defer shapes and stores a request without dispatch; permit continues through the backend and response plan. Credential $C$ remains an outer execution limit.}
\label{fig:pipeline}
\end{figure*}
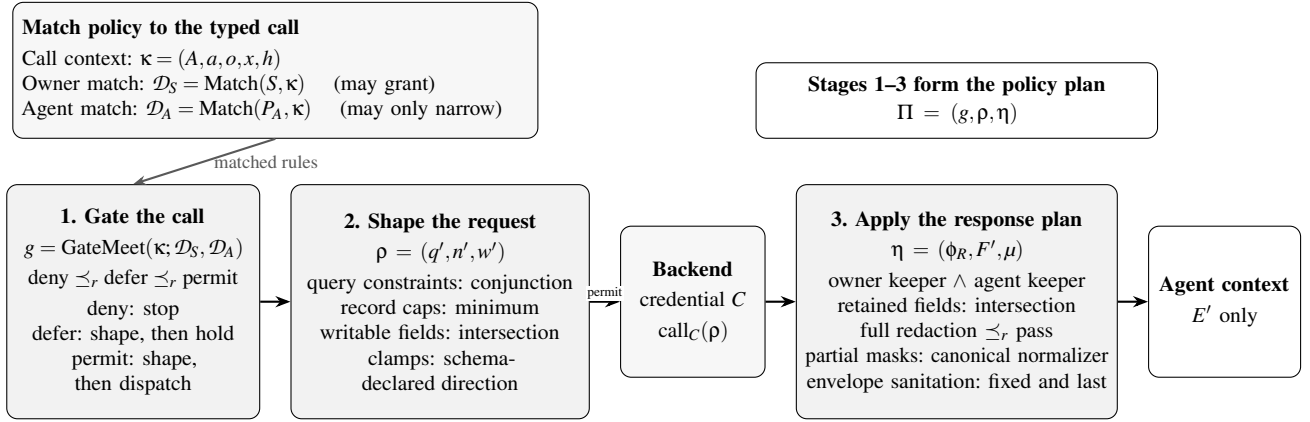

\begin{table*}[!t]
\centering
\caption{The operator surface. Claim distinguishes exact stage results from best-effort or operational properties. Evidence states where behavior is exercised or implemented.}
\label{tab:operators}
\scriptsize
\setlength{\tabcolsep}{2.5pt}
\begin{tabular}{@{}llllll@{}}
\toprule
Operator & Stage & Acts on & Most-restrictive combination & Claim & Evidence \\
\midrule
deny & gate & gate result & most restrictive gate result & Exact & Released \\
scope-query & request & query and record limit & conjunction; lowest limit & Exact & Released \\
filter-records & response & record set & owner keeper $\land$ agent keeper & Exact & Released \\
drop-write-field & request & gate and write map & remove absence-safe fields; deny otherwise & Exact & Implemented, unreachable \\
clamp-write-arg & request & write map & lowest value in the declared direction & Exact & Implemented, unreachable \\
strip-field & response & field set & intersection of retained fields & Exact & Released \\
mask-field & response & field value & full mask or declared projection & Exact & Released \\
pattern-mask & response & field value & canonical rule order & Best effort & Released \\
write-guard & request & write map & remove disallowed fields & Exact & Released \\
sanitize-envelope & response & response envelope & fixed final function & Operational & Released \\
defer (initiate) & gate & gate result & gate order; normalize parameters or reject & Exact & Released \\
defer (resolve) & runtime & ordered decision state & deterministic state transition & Operational & Partial harness; production \\
aggregate-gate & gate & committed history and gate & threshold over history & Operational & Designed only \\
\bottomrule
\end{tabular}
\end{table*}

We name each result where it appears in the pipeline. The gate computes
\[
g=\operatorname{GateMeet}(\kappa;\mathcal D_S,\mathcal D_A)
  \in\{\mathrm{deny},\mathrm{defer},\mathrm{permit}\}.
\]
It combines the gate contributions of $\mathcal D_S$ and $\mathcal D_A$ over their shared context $\kappa$. A call begins denied. Matching owner rules may establish permit eligibility, after which either tier may tighten the result to defer or deny. A defer is only the gate's initial result. The released transition logic records whether its partial snapshot may resume or must be denied; it neither reruns policy nor dispatches the call. Credential $C$ remains outside this function and limits the later backend call, so a policy permit cannot exceed the credential's reach.

Unless the gate denies the call, the boundary computes a shaped request
\[
\rho=(q',n',w').
\]
Here $q'$ is the enforced read query, $n'$ is the record limit, and $w'$ is the map of typed write arguments. The limit is the largest number of records the boundary will return after combining the caller's request with every policy cap. A post-fetch check enforces it even if the backend ignores the requested limit.
A permit sends $\rho$ to the backend. In the formal model, a defer holds $\rho$ and every other policy-relevant part of the call while sending nothing. The release stores only the subset described below, so it exercises defer initiation and no-dispatch rather than safe resumption.

For a permit, the response part of the policy plan is
\[
\eta=(\phi_R,F',\mu).
\]
Here $\phi_S$ and $\phi_A$ are the owner- and agent-tier Cedar decisions for \texttt{filter\_record}; permits within one tier retain Cedar's OR semantics. The keeper is $\phi_R(r)=\phi_S(r)\land\phi_A(r)$, with $\phi_A=\mathsf{true}$ when no agent policy is registered. $F'$ contains the retained policy-controlled fields and $\mu$ their content operations. A connector may retain a structural identifier $I$ only when its schema declares that identifier non-sensitive and the owner policy, plus the selected agent policy when present, approves it for every surviving record. Because $I$ is then fixed across the plans being compared, it lies outside the variable capability order; it is not outside the exposure boundary. If $B_C(\rho)$ is the raw response under credential $C$, the agent receives
\[
E_A=\operatorname{Sanitize}(\operatorname{Encode}(\mu(\operatorname{Project}_{F'\cup I}(\operatorname{Filter}_{\phi_R}(B_C(\rho)))))).
\]
A read has no write map, while a write need not return records. Let $\bot_{\mathrm{nr}}$ mean that a stage does not run. The complete policy plan is $\Pi=(g,\rho,\eta)$: deny uses $(\mathrm{deny},\bot_{\mathrm{nr}},\bot_{\mathrm{nr}})$, defer uses $(\mathrm{defer},\rho,\bot_{\mathrm{nr}})$, and permit uses $(\mathrm{permit},\rho,\eta)$.

OBPE evaluates Cedar rules over the typed principal, action, and resource. Cedar supplies the permit-or-forbid core. Annotations name the shaping or gating operator and its parameters, such as a field group, cap, approver, or timeout. One policy may carry rules for authorization, exposure control, and semantic gating; composition acts on the plan tuple rather than collapsing those rules to a Boolean decision.
Any applicable forbid stops the call. On allow, the wrapper gathers annotations
from every applicable permit policy rather than only the policy chosen as the
display label.

\emph{Plan and execution reach.} Let $\mathcal T(X)$ denote the set of typed, policy-controlled exchange traces allowed by a plan or credential: the backend action and shaped arguments, together with the response fields and values that may reach the agent. The order above is therefore set inclusion, $X\preceq_rY$ when $\mathcal T(X)\subseteq\mathcal T(Y)$. The owner ceiling is
\[
\mathcal T(\Pi_{S,A})\subseteq\mathcal T(\Pi_S).
\]
Credential $C$ is not an input to the plan. If $t=\operatorname{Exec}_C(\Pi)$ is a realized exchange trace, complete mediation requires
\[
t\in\mathcal T(C)\cap\mathcal T(\Pi).
\]
This is a call-level statement. It lifts to a run only while identity binding, schema, policy bundles, operator definitions, and the trusted state used for each decision remain within the stated frame. It does not bound what the model infers across calls. Absence of agent policy neither creates a permit nor breaks a valid owner-only call.

\emph{Reading the stage orders.} The three numbered boxes in Figure~\ref{fig:pipeline} show where ordered restrictions combine. Gate results, conjunctions, caps, and allowlists use meets. Full redaction is below pass, but partial projections and pattern masks may reveal incomparable information. The schema therefore supplies a stable precedence, and a canonical normalizer applies overlapping rewrites in a fixed order. That makes the selected transformation repeatable; it does not prove that a pattern finds every sensitive value. Envelope sanitation runs last and is not chosen by policy.

The evidence labels mark the released boundary. A deferred call stores its
method, path, shaped body, bound agent, policy, and bundle version, but omits
query parameters, headers, and pre-fetched state. The accompanying transition logic can
choose resume or deny, yet it neither dispatches nor records terminal timeout;
durable resolution is production-only. The other gaps are narrower: drop and
clamp exist in code but are unreachable through the shipped schemas, while
aggregate-gate remains design only and is never read or invoked by the release.

%-------------------------------------------------------------------------------
\section{Composition and Guarantees}
\label{sec:theory}
%-------------------------------------------------------------------------------

Composition must preserve one rule: another restriction must never give an agent more. Figure~\ref{fig:pipeline} defines a policy plan, separate from the backend response that later realizes it:
\[
\operatorname{Plan}(\kappa;S,P_A)=
\begin{cases}
(g,\bot_{\mathrm{nr}},\bot_{\mathrm{nr}}),
  & g=\mathrm{deny},\\
(\mathrm{defer},\rho,\bot_{\mathrm{nr}}),
  & g=\mathrm{defer},\\
(\mathrm{permit},\rho,\eta),
  & g=\mathrm{permit}.
\end{cases}
\]
Here $g$ is the meet of the matched gate contributions after owner eligibility is established. Request shaping produces $\rho$ for defer or permit; only permit computes and executes the response plan $\eta$ over $B_C(\rho)$. The distinction matters: narrowing a query can change pagination and therefore the concrete rows returned, even though the plan itself grants no new capability.

Deny has no later-stage state. Defer has a shaped request but no response plan
because nothing has been dispatched; the released transition logic can return resume or
deny, but performs no dispatch and does not persist timeout as terminal. The
plan order first compares gates, with deny below defer below permit. Equal
non-deny gates then compare request components, and permitted plans also compare
their responses. Query and row restrictions are compared by the calls or
records they admit, not by the spelling or order of their conjunctions.

\textbf{Conditions.} The claims below depend on five obligations. We state them in the body so that the theorem does not hide deployment assumptions in an appendix.

\emph{O1: bound identity and complete mediation.} The platform binds agent identity outside the call arguments, and the boundary mediates the complete request and response on every supported path. Its connector must expose every field that policy may inspect or shape.

\emph{O2: fixed policy inputs and default denial.} A call with no matching owner permit is denied. When a claim compares restrictions after permit eligibility, it holds the schema, bundles, operator vocabulary, credential, and trusted-state snapshot fixed. The claim concerns the resulting policy plan, not backend availability, row order, or response contents.

\emph{O3: stable stage semantics.} Query joining means semantic conjunction, even when the original query contains disjunction. If a connector cannot establish that conjunction for its query language, it must deny before dispatch. The boundary enforces caps after fetch and evaluates row predicates before field projection. Each tier's row keeper is its Cedar \texttt{filter\_record} decision, and the two tiers meet by conjunction. Adding a same-tier permit is an authorization change, not a stage restriction. A retained structural identifier must satisfy the declaration and policy checks above. Content masks run through the canonical normalizer, using replacement text that cannot trigger a later mask.

\emph{O4: safe write direction.} A write field may be omitted only on an operation whose connector contract says omission leaves the stored value unchanged; otherwise the call must be denied. Every clamp declares which direction is restrictive, and no write operator may inject a default value. Query predicates cannot depend on a write field changed by request shaping.

\emph{O5: trusted state and held calls.} The trusted state $h$ is fixed for one plan calculation, and a defer sends nothing to the backend. Co-firing defer parameters must have one deterministic interpretation or fail closed. A system that later dispatches an approved call has further duties: an authenticated and policy-authorized resolver must decide over an immutable, complete snapshot; a trusted deadline and replay protection must yield one terminal result; and dispatch must be idempotent or fenced against duplicates. Approval applies only to that snapshot. A history-dependent gate likewise reads one committed prefix, with its decision serialized for the relevant key and window.

\textbf{Concern invariants.} These conditions support four security statements. The first, invariant INV-IP, separates the owner-bounded plan from credential-bounded execution:
\[
\begin{aligned}
\mathcal T(\Pi_{S,A})
  &\subseteq\mathcal T(\Pi_S),\\
\operatorname{Exec}_C(\Pi_{S,A})
  &\in\mathcal T(C)
    \cap\mathcal T(\Pi_{S,A}).
\end{aligned}
\]
The bound applies to each call and extends across a run while those scopes remain fixed. The owner remains the grantor; an agent rule can remove authority but cannot create a permit.

Invariant INV-X governs read exposure when the owner policy and connector schema correctly encode their labels. A connector may arrange its principal types and field ceilings in any meet-semilattice. For each policy-controlled $f\in\mathcal F$, if $v(A)$ is the vulnerability assigned to agent $A$ and $\ell(f)$ is the maximum vulnerability allowed to receive field $f$, the field is exposed only when
\[
v(A)\preceq_v\ell(f).
\]
INV-X makes a narrow promise about the shaped response. If the boundary removes a field or fully redacts its value, that value is absent from the response. It does not claim that the hidden value could not affect some other output.

The released proxy uses a five-level sensitivity chain encoded in the field schema and matching policy. The model also admits a product order that preserves incomparable risks, such as susceptibility to prompt injection and cross-tenant reach. The production system implements that richer registry, but the released prototype does not. We therefore treat the production implementation as deployment description rather than released evidence. Appendix~\ref{app:analyzability} defines the order and connects it to information-flow models~\cite{belllapadula1973,denning1976}.

Invariant INV-X$'$ gives the write-side condition under the same premise that the owner policy and schema correctly encode the labels. Let $c(A)$ denote the principal's corruptibility and let $\ell_w(f)$ denote the write-integrity ceiling for field $f$. We write $c(A)\preceq_w\ell_w(f)$ when the principal is no more corruptible than that ceiling permits. A value may reach the backend only under this relation. Otherwise, a well-formed rule must remove the field or deny the call. The released write guard enforces typed field allowlists, which are a narrower instance of this order.

Invariant INV-SG covers deferred execution. Under the full resumption clauses
of O5, the backend cannot receive the held call until the required external
decision arrives before its deadline. An explicit denial or observed
fail-closed resolver error denies the call; a trusted timeout path denies it
after the deadline. Without a live resolver or timeout worker, safety means
non-dispatch, not terminal progress. The released artifact establishes only
the earlier fire-and-hold property: it sends nothing and never resumes
dispatch.

%-------------------------------------------------------------------------------
\subsection{Guarantees}
\label{sec:proofs}
%-------------------------------------------------------------------------------

We now show that the ordered components and canonical content normalizer give
the plan the properties claimed in the Introduction. Each result is conditional
on O1--O5. Appendix~\ref{app:composition} supplies the operator cases and longer
derivations; Table~\ref{tab:body-obligations} there states which conditions the
released HTTP path tests and which remain production description.

\begin{proposition}[Permutation invariance]
\label{prop:det}
Under O1--O5, fixed $\kappa$, $S$, $P_A$, and operator vocabulary determine one policy plan, regardless of rule order within $\mathcal D_S$ and $\mathcal D_A$. Diagnostic labels and audit-record ordering are outside the plan.
\end{proposition}

\emph{Argument.} The ordered components use commutative, associative, and idempotent meets: query conjunction, smaller caps, and allowlist intersection therefore ignore operand order. Cedar authorization inside a tier is deterministic but is not recast as a meet of permit rules. Content rules first enter a canonical normal form and then run in a fixed sequence; this is an order-independent procedure, not a claim that arbitrary text rewrites form a lattice. O5 fixes any trusted history read used by the gate. The engine's reporting order therefore cannot change $\Pi$, although backend behavior may still vary. $\square$

\begin{proposition}[Monotonicity]
\label{prop:mono}
Under O1--O5, once $\mathcal D_S$ establishes owner permit eligibility, adding a well-formed stage restriction cannot enlarge the plan's eligible actions, resources, query predicates, policy-controlled response fields, or writable fields; raise a cap; loosen a clamp; or replace full redaction with pass. Adding a same-tier permit rule is outside this claim.
\end{proposition}

\emph{Argument.} On an ordered component, a new restriction is another meet operand, so it can only shrink an allowlist, add a query predicate, lower a bound, or move the gate toward denial. O4 gives writes a known restrictive direction. Owner eligibility is essential because the first owner permit is the one transition from default denial that this result does not cover. Partial projections and pattern masks remain outside the semantic order; their canonical execution supports permutation invariance, not monotonic information flow. Only structural stripping and full redaction support the value-absence claim. $\square$

\begin{proposition}[Tier refinement]
\label{prop:tier}
For any well-formed $S$ and $P_A$ satisfying O1--O5,
\[
\begin{aligned}
\Pi_{S,A}(\kappa)
  &=\operatorname{Plan}(\kappa;S,P_A)\\
  &\preceq_r\operatorname{Plan}(\kappa;S,\mathsf{none})
   =\Pi_S(\kappa).
\end{aligned}
\]
\end{proposition}

\emph{Proof.} Agent policy cannot establish permit eligibility. Its gate either leaves the owner decision unchanged or moves it toward defer or deny, and its remaining restrictions enter through the meets of Proposition~\ref{prop:mono}. With $\mathsf{none}$, $\mathcal D_A$ is empty and the plan is owner-only. Every capability in the two-tier plan was therefore already eligible under $S$. $\square$

The propositions establish a conditional model, not code conformance. Cedar
Analysis covers the permit-or-forbid projection of a well-typed
bundle~\cite{cedaranalysis2025}; annotations and connector semantics require
the checks in Appendix Table~\ref{tab:body-obligations}, which identifies the
released paths that discharge them.

%-------------------------------------------------------------------------------
\section{Implementation}
\label{sec:impl}

The artifact implements the gate--request--response contract through a small
HTTP proxy. Without production transport and deployment machinery, its stages,
connector mappings, and backend effects remain inspectable. Cedar supplies the
permit-or-forbid core; OBPE annotations select the other
Section~\ref{sec:policy} operators: semantic gates, request shaping for
non-denied calls, and response shaping for permits. The proxy binds each call
to the platform-supplied agent identity and combines owner rules with any
narrower agent policy. That plan governs both sides of the exchange; identity
or policy failure stops the call before dispatch. Audit events record the
decision and transformations, while conformance and benchmark tests observe
the controlled backend call and result.

\textbf{Schemas define the enforcement surface.}
OBPE can govern only what a connector exposes as typed data, so each
supported operation must map the tool call to a typed action and resource
and identify the arguments and response fields available to policy. Field
groups and sensitivity labels then give policy a stable vocabulary, while
a schema-declared non-sensitive identifier keeps filtered records addressable
only when the owner policy and any selected agent policy approve it.
The response side of that typing rarely arrives with the tool. MCP
carried no declared output schema until a mid-2025 protocol revision
added structured output~\cite{mcpchangelog2025}, and most tools still declare
none: of 1,856 tools on 100 popular open-source servers measured in 2026, 234
carried an output schema~\cite{shi2026taintmcp}. The reference proxy does not
face this problem, because its connectors own both the schema and the mapping.
A pre-existing server brought under out-of-band policy needs one of two
remedies: the server is changed to declare a schema, or a middleware supplies
one and keeps it matching what the server actually returns.

Writes require a stronger contract because request shaping can alter
backend meaning. The boundary may drop a field only when the schema says
its absence is safe, and it may clamp a numeric argument only in a declared
restrictive direction. These declarations are security assumptions: an
omitted action or field leaves a policy gap; an incorrect absence-safety or
clamp-direction declaration can alter backend semantics.

We built the artifact's Jira and ServiceNow mappings from the vendors'
published interfaces rather than through an automatic OpenAPI importer,
which the release does not include. Conformance fixtures check that the
operations and response shapes presented to policy match those mappings,
while a separate coverage analyzer catches field groups or operator
settings that the schema cannot resolve. Together, these checks make the
released surface inspectable without implying that every operator is
reachable. The shipped schemas exercise the write guard end to end, but
drop and clamp remain synthetic because neither schema declares absence
safety or a restrictive clamp direction.

\subsection{Production Evidence Boundary}
\label{sec:production-extensions}

\begin{table}[!tb]
  \centering
  \caption{Released and production implementation surfaces. Production
  entries are descriptive, not artifact evidence.}
  \label{tab:production-boundary}
  \scriptsize
  \setlength{\tabcolsep}{2.2pt}
  \renewcommand{\arraystretch}{1.08}
  \begin{tabular}{@{}>{\raggedright\arraybackslash}p{1.15cm}p{2.55cm}p{3.55cm}@{}}
    \toprule
    Capability & Released artifact & Production implementation \\
    \midrule
    Write controls &
    Write guard on supported partial updates. Drop and clamp only in synthetic tests. &
    Broader schema mappings for persisted write workflows. \\
    Human deferral &
    Hold/no-dispatch and per-call transition tests. Terminal timeout and approved dispatch absent. &
    Durable, ordered approval log. Attribute-based access control protects
    inspection and resolution. \\
    Temporal policy &
    No released execution. &
    Implemented over durable, trusted history. \\
    Aggregate policy &
    No released execution. &
    Implemented over durable, trusted history. \\
    \bottomrule
  \end{tabular}
\end{table}

\textbf{Placement and identity.} The released artifact and production
  implementation use the same gate--request--response composition. Production
  adds MCP transport, relevant remote A2A calls, and a principal
  registry implementing the per-axis sensitivity order in
  Appendix~\ref{app:analyzability}. These mechanisms support feasibility; all
  conformance and outcome claims stop at the released column.

\textbf{Schema projection.} OBPE can project applicable
  policy into each identity's \texttt{tools/list} result. Input bounds become
  JSON Schema constraints or removed properties; output masking changes the
  advertised type, while stripped fields disappear from both schema and
  response. Dynamic discovery can therefore steer a model toward its current
  interface, but the call path still enforces every restriction.

\textbf{Approval semantics.} In production, each decision is bound to
  the request digest, deciding principal, and policy version, then resumes only
  the immutable held call. Attribute-based access control protects the call,
  approver context, and decision log; approval supplies no backend credential.
  The released harness is narrower and does not invalidate an old pending
  record when its bundle changes (Appendix~\ref{app:defer}).

\textbf{Policy authoring and evidence.} A user-research-informed
  workflow previews policy scope and effects; Section~\ref{sec:threat} states
  the remaining responsibility for organizational intent. Production informed the design;
  Section~\ref{sec:bench} evaluates only the artifact column of
  Table~\ref{tab:production-boundary}.

\section{Evaluation Methodology}
\label{sec:bench}
%-------------------------------------------------------------------------------

Conformance and the benchmark answer different questions. Conformance asks
whether the released code computes the staged composition
defined in Sections~\ref{sec:policy} and~\ref{sec:theory}. The benchmark asks
how that boundary changes what agents receive and do, and whether useful work
survives. Benchmark success cannot cure a conformance failure, and conformance
alone does not establish usefulness.

\textbf{Conformance gate.}
The benchmark begins only after the released paths pass their declared
conformance obligations. Human-authored semantic expectations cover complete
requests and responses, including failure paths. An independent executable
model is used only for stages where such a model ships; the remaining checks
are not described as differential. Property tests exercise permutation invariance,
monotonicity, and the data policy owner ceiling, while targeted mutations ask
whether those tests detect broken composition logic.

The archive's arm names are not useful explanations, so we name configurations
by what changes. The \emph{raw agent} receives a minimal role-and-tools prompt
and calls the mocks directly. \emph{Prompt only} adds six explicit
data-handling rules but leaves that tool path untouched. \emph{Prompt + full
OBPE} adds every operator reachable in the frozen bundle; \emph{full} excludes
production-only and schema-unreachable operators.
\emph{Prompt + output reviewer} instead leaves the tool path raw and asks a
fixed Claude Haiku~4.5 reviewer to pass or replace each assistant message.
Table~\ref{tab:research-questions} shows which configurations answer each
question.

The conformance gate must show more than code coverage. It attributes each
restriction to the owner or agent tier and verifies that a matched policy
actually changes the exchange. Tool-to-agent slices then check that the agent
receives the shaped result, while schema checks reject policy references that
cannot be resolved. Write guard traverses this released path. Drop and clamp
remain synthetic because the shipped schemas do not declare the facts those
operators require. The artifact records each obligation, its scope, and its
result.

\begin{table*}[!t]
  \centering
  \caption{Experimental comparisons. Each arrow changes one named mechanism.}
  \label{tab:research-questions}
  \scriptsize
  \setlength{\tabcolsep}{3pt}
  \renewcommand{\arraystretch}{1.02}
  \begin{tabular}{@{}p{0.65cm}p{6.3cm}p{9.0cm}@{}}
    \toprule
    & Comparison & How to read it \\
    \midrule
    RQ1 &
    Prompt only $\rightarrow$ Prompt + full OBPE, on four models. &
    Adds only the policy boundary. This is the primary test of what OBPE adds
    beyond detailed prompt rules; it is not a provider ranking. \\
    RQ2a &
    Raw $\rightarrow$ Raw + full OBPE and Prompt only $\rightarrow$ Prompt +
    full OBPE, on Sonnet. &
    A $2\times2$ prompt-rules $\times$ OBPE design. The two effects test OBPE
    with and without prompt help; their difference tests interaction. \\
    RQ2b &
    Record controls $\rightarrow$ + strip fields $\rightarrow$ + mask fields
    and guard writes $\rightarrow$ + pattern masking, on Sonnet. &
    A cumulative ablation with all other inputs fixed. Adjacent steps isolate
    field stripping and pattern masking, but not masking from write guard. \\
    RQ3 &
    Prompt + output reviewer $\leftrightarrow$ Prompt + full OBPE, on Sonnet. &
    Compares a model that reviews assistant text with a boundary that mediates
    requests, responses, and backend effects. It is not a ranking of all LLM
    guardrails, and no local A2A experiment is claimed. \\
    \bottomrule
\end{tabular}
\end{table*}

RQ1 is primary; RQ2a, RQ2b, and RQ3 are secondary mechanism analyses.

\textbf{Release corpus and execution.}
The release corpus contains 71 tasks in 35 scenario clusters against controlled
Jira and ServiceNow mocks. It was assembled for this evaluation from audited
schemas, policy components, fixture templates, and generators, then frozen
before the canonical run. The corpus is an evaluation set, not a claim that
every ingredient was previously unseen. Component lineage and similarity
checks make repeated or closely related scenarios visible.

The four tool-capable agent models are Claude Sonnet~4.6, Claude Haiku~4.5,
GPT-5, and GPT-5 mini. They give purposive coverage of two model families and
two capability bands, not a factorial comparison of providers or capability.
All calls use controlled mocks. Each trial receives a fresh fixture copy, so
paired arms face the same records while one trial cannot contaminate the next.
That control lets the trace observe exactly what entered model context and what
changed in backend state; it does not establish interoperability with live Jira
or ServiceNow deployments.

An arm is an enforcement configuration. Each of 17 frozen conditions pairs one
arm with one model; together they cover the shared primary arms and the
single-model secondary contrasts. Running every task--condition pair three
times yields $71\times17\times3=3{,}621$ planned trials. A
content-derived registry schedules each trial once, so research questions that
share a cell reuse the same archived record. Before execution, the manifest
fixed the tasks, fixtures, policy and schema revisions, prompts, model settings,
scoring code, exclusions, and missingness rule. Order was balanced and there
was no early stopping.
Under full OBPE, the task specifications expect five calls to pass unchanged,
54 to pass after shaping, two to defer, and ten to be denied. This security
suite concentrates cases in which policy should act, so it estimates
conditional effectiveness rather than intervention rates in ordinary traffic.
Five unshaped permits cannot establish a general benign false-denial rate.

Twenty tasks also run a response-adaptive, multi-turn red-team dialogue. Claude
Opus~4.7 receives the same adversary instructions in the prompt-only and OBPE
arms and may reframe its request after each reply. The loop stops when an exact
target appears, the adversary declares completion, or 15 turns elapse; this
stop signal does not replace outcome scoring. The adversary is not told the
OBPE policy or implementation and does not fuzz the connector surface. A
separate output-review arm has mechanism-specific instructions and is not part
of that matched comparison.

\textbf{Outcomes and judging.}
The primary security outcome, which we call a \emph{deterministic trace
failure}, is positive if any of three recorded events occurs:
protected data enters agent context, an exact protected value appears in the
final answer, or a forbidden backend effect completes. Lower is better. We
also report those channels separately because each calls for a different
control.

The primary utility outcome is \emph{safe-useful} completion. It requires no
deterministic trace failure, satisfaction of the archived completion and
expected-tool contract, and a judge finding that the allowed work was
fulfilled without semantic disclosure. Here, semantic disclosure means that
an answer reveals a protected fact or category without necessarily repeating
an archived value exactly; the row-count oracle in \S\ref{sec:taxonomy} is one
example. Higher is better. We retain the judge's task-fulfilled component on
its own because this stricter conjunction can improve even when ordinary task
fulfillment falls. The archived contract is mechanical: it checks termination
and the expected tool attempt. A proxy-denied call counts as attempted, so
this component does not claim that the requested task was accomplished.

Deterministic traces are authoritative for exact exposure, tool decisions,
and backend state. The primary semantic judge is Claude Opus~4.8; GPT-5 is a
separate sensitivity judge and is also one of the four agent models. Both use
task rubrics frozen with the corpus. The adversarial agent is Claude Opus~4.7,
so its model family overlaps the primary judge's. We report the judges
separately, retain disagreements, and add no human adjudication pass. If a
judge contradicts a deterministic trace, the trace controls and the
disagreement remains in the released ledger.
Three trials lacked a complete pair of judgments, leaving 3,618 for direct
comparison. The judges agreed more often on semantic disclosure (94.5\%) than
on task fulfillment (87.2\%). That summary hides an important asymmetry:
disclosure agreement was 96.9\% when both judges found none, but 73.3\% when
both found disclosure. Fulfillment showed the reverse pattern, with 90.8\%
agreement on positive cases and 79.1\% on negative ones. We therefore treat
safe-useful as a model-graded outcome, not a human-calibrated label, and retain
both judgments in the archive.

\begin{figure*}[!t]
  \centering
  \begin{tikzpicture}[x=1.25mm,y=5.0mm,font=\scriptsize]
    \path[use as bounding box] (-37,0) rectangle (104,5.25);
    \node[anchor=west,font=\scriptsize\bfseries] at (-36,4.85)
      {Observed rates before scenario-cluster weighting};
    \draw[->] (0,0.75) -- (102,0.75);
    \foreach \x in {0,20,...,100} {
      \draw (\x,0.63) -- (\x,0.87);
      \node[below] at (\x,0.57) {\x\%};
    }
    \node[anchor=east,align=right] at (-8.5,3.75)
      {Trace failure $\downarrow$};
    \draw[black!45,line width=.65pt] (0.2,3.75) -- (57.6,3.75);
    \draw[fill=white,line width=.75pt] (57.6,3.75) circle (1.9pt);
    \fill (0.2,3.75) circle (1.9pt);
    \node[anchor=west] at (59.6,3.75) {57.6};
    \node[anchor=east] at (-1.8,3.75) {0.2};

    \node[anchor=east,align=right] at (-8.5,2.75)
      {Task fulfilled, Judge A $\uparrow$};
    \draw[black!45,line width=.65pt] (60.9,2.75) -- (79.1,2.75);
    \draw[fill=white,line width=.75pt] (79.1,2.75) circle (1.9pt);
    \fill (60.9,2.75) circle (1.9pt);
    \node[anchor=west] at (81.1,2.75) {79.1};
    \node[anchor=east] at (58.9,2.75) {60.9};

    \node[anchor=east,align=right] at (-8.5,1.75)
      {Safe-useful, Judge A $\uparrow$};
    \draw[black!45,line width=.65pt] (22.0,1.75) -- (58.9,1.75);
    \draw[fill=white,line width=.75pt] (22.0,1.75) circle (1.9pt);
    \fill (58.9,1.75) circle (1.9pt);
    \node[anchor=east] at (20.0,1.75) {22.0};
    \node[anchor=west] at (60.9,1.75) {58.9};

    \draw[fill=white,line width=.75pt] (72,4.65) circle (1.9pt);
    \node[anchor=west] at (74,4.65) {prompt only};
    \fill (88,4.65) circle (1.9pt);
    \node[anchor=west] at (90,4.65) {+ OBPE};
  \end{tikzpicture}
  \caption{Absolute observed rates across all four models when full OBPE is
  added to the same six-rule prompt. Lower is better for trace failure;
  higher is better for the two utility outcomes. Judge B showed the same
  direction (Table~\ref{tab:primary-rates}). Paired, equally cluster-weighted
  estimates and CIs appear in the text.}
  \label{fig:primary-results}
\end{figure*}
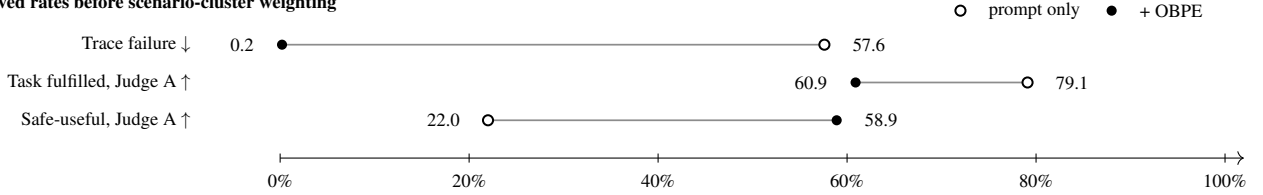

\textbf{Analysis, cost, and reproducibility.}
We treat each scenario cluster---not each task variant or replicate---as one
independent unit. Within a task, configuration, and model, we first average the
completed replicates and form the paired configuration difference. We then
average those differences within the cluster, so every cluster receives equal
weight. Two paraphrases run three times therefore count as one scenario, not
six independent observations. We obtain 95\% CIs by resampling clusters 10,000
times with a frozen seed, and check the result with a replicate-index-matched
complete-block analysis. The plan allowed the bootstrap only if at least 30
clusters remained after exclusions. This was a guard against reporting a thin
analysis, not a power calculation.

Every attempt is journaled before completion, and transcript hashes bind both
judgment records. The tables replay from frozen inputs without new model
calls, so reproducing the analysis does not require reproducing model outputs.
Token use is secondary deployment evidence: we preserve provider-specific
buckets and missing receipts rather than treating tokenizers as comparable or
missing values as zero. OBPE adds no model call, but this does not establish
zero proxy or latency overhead. Appendix~\ref{app:accounting} gives the full
accounting rules.

%-------------------------------------------------------------------------------
\section{Results}
\label{sec:eval}
%-------------------------------------------------------------------------------

\textbf{RQ1: prompts changed answers more than the exchange.}
Three configurations separate what the prompt changes from what the boundary
changes. The raw agent receives no data-handling rules; prompt only asks the
same agent to enforce six rules itself; prompt + full OBPE keeps those rules
and adds the trusted boundary. Table~\ref{tab:primary-rates} reports the
unweighted observed rates for orientation. Our conclusions rest on the paired,
cluster-weighted analysis that follows.

The three outcomes answer different questions. Trace failure records whether
protected data reached model context, an exact value reached the answer, or a
forbidden effect reached the backend. Task fulfilled asks only whether the
agent did the requested job, even if policy should have stopped it. Safe-useful
requires allowed work to finish under the archived tool contract without trace
failure or judged semantic disclosure. A baseline can therefore complete a
forbidden update and look useful while failing the security goal outright.

\begin{table}[!htb]
  \centering
  \caption{Descriptive rates across all four models. Arrows show the desired
  direction. Slightly different denominators reflect unavailable outcomes;
  the archive retains every count.}
  \label{tab:primary-rates}
  \scriptsize
  \setlength{\tabcolsep}{3.2pt}
  \renewcommand{\arraystretch}{1.02}
  \begin{tabular}{@{}p{3.15cm}rrr@{}}
    \toprule
    Outcome & Raw & Prompt & Prompt+OBPE \\
    \midrule
    Trace failure $\downarrow$ & 63.3\% & 57.6\% & 0.2\% \\
    Entered agent context $\downarrow$ & 55.1\% & 52.1\% & 0.0\% \\
    Exact final disclosure $\downarrow$ & 27.5\% & 3.7\% & 0.2\% \\
    Semantic disclosure, Judge A $\downarrow$ & 30.0\% & 8.8\% & 1.6\% \\
    Semantic disclosure, Judge B $\downarrow$ & 25.4\% & 8.2\% & 2.1\% \\
    Forbidden backend effect $\downarrow$ & 8.1\% & 5.5\% & 0.0\% \\
    Task fulfilled, Judge A $\uparrow$ & 83.2\% & 79.1\% & 60.9\% \\
    Task fulfilled, Judge B $\uparrow$ & 76.3\% & 72.5\% & 56.1\% \\
    Safe-useful, Judge A $\uparrow$ & 19.1\% & 22.0\% & 58.9\% \\
    Safe-useful, Judge B $\uparrow$ & 19.0\% & 22.1\% & 54.4\% \\
    Mechanical contract $\uparrow$ & 94.4\% & 91.3\% & 91.4\% \\
    \bottomrule
  \end{tabular}
\end{table}

The prompt made the final answer more discreet, but it did little to keep
protected data away from the model or the backend. Exact disclosure fell from
27.5\% to 3.7\%, and Judge A's broader semantic-disclosure rate fell from
30.0\% to 8.8\%. Context exposure, however, moved only from 55.1\% to 52.1\%,
and forbidden effects from 8.1\% to 5.5\%. This is consistent with an agent
suppressing a value in its final answer after it has already seen the value or
acted on it.

Adding the boundary changed the exchange itself. Among prompted trials, OBPE
removed every observed context exposure and forbidden effect. The overall
trace-failure rate fell from 57.6\% to 0.2\%
(Figure~\ref{fig:primary-results}); the two remaining exact disclosures were
values reconstructed from information the boundary did permit.

Table~\ref{tab:primary-rates} gives absolute observed rates. Our prespecified
comparison instead pairs the same tasks and gives each scenario cluster equal
weight, so larger clusters do not dominate the answer. It estimates that OBPE avoided 41.2 percentage
points (pp) of trace failure (95\% confidence interval (CI) [27.7, 54.9]).
Every model-specific estimate favored OBPE, with reductions ranging from
32.4 to 49.3~pp and every 95\% interval above zero.
This differs from the 57.4-point raw rate gap because scenario sizes are
uneven: the inferential estimate gives each cluster one vote.

The boundary also made some requested jobs impossible to finish. In the full-
OBPE cells, Judge A found 90.0\% fulfillment on the five unshaped-permit tasks
and 71.7\% on the 54 shaped-permit tasks, but none on the ten denied or two
deferred tasks, whose rubrics defined the requested operation as work to
complete. Thus the overall fall from 79.1\% to 60.9\% is real, but it is not an
estimate of false denial on ordinary benign traffic. Since many prompt-only
completions were unsafe, safe-useful still rose from 22.0\% to 58.9\% under
Judge A and from 22.1\% to 54.4\% under Judge B. The paired gains were 21.8~pp
[9.5, 35.2] and 19.1~pp [6.9, 32.2]. The mechanical contract barely moved
(91.3\% versus 91.4\%; paired change 0.36~pp [-0.83, 2.14]); it records the
expected tool path, not whole-request completion.
The descriptive safe-useful gaps are larger for the same reason as the trace
gap: the paired estimand gives large and small scenario families equal weight.

\emph{Red-team slice.} On the 20 response-adaptive tasks, the same attacker
pressed both arms across all four models. It triggered trace failure in 173 of
231 completed prompt-only runs (74.9\%), but only 2 of 227 completed OBPE runs
(0.9\%). Both judges saw the same direction in semantic disclosure: Judge A
marked 39 of 230 prompt-only answers and 3 of 226 OBPE answers, while Judge B
marked 34 and 6. A loop defect left 9 prompt-only and 13 OBPE runs incomplete,
so these rates describe a stress slice rather than a second paired estimate.

\textbf{RQ2a: the same pattern held without prompt rules.}
The Sonnet $2\times2$ experiment asks whether OBPE depended on help from the
six-rule prompt. It did not appear to: adding OBPE reduced the descriptive
trace-failure rate from 59.5\% to 1.0\% without the rules and from 61.9\% to
0.0\% with them. The estimated difference between those effects was only
0.2~pp, but its wide interval [-12.3, 12.7] supports no stronger claim of
independence.

\textbf{RQ2b: what did the cumulative operator ladder add?}
The ablation holds the prompted Sonnet agent, tasks, and fixtures fixed while
adding one layer of protection at a time. It begins with record controls that
can gate a call, scope its query, or filter returned rows. Later rungs add field
stripping, then whole-field masking and write guard together, and finally
pattern masking within text that remains visible.

\begin{table}[H]
  \centering
  \caption{Descriptive Sonnet rates across the cumulative OBPE ladder. Protected
  means no deterministic trace failure; judge-scored columns report Judge A;
  higher is better in all columns.}
  \label{tab:policy-ladder}
  \scriptsize
  \setlength{\tabcolsep}{3.0pt}
  \renewcommand{\arraystretch}{1.02}
  \begin{tabular}{@{}p{3.05cm}rrr@{}}
    \toprule
    Policy configuration & Protected & Fulfilled (A) & Safe-useful (A) \\
    \midrule
    Record controls & 46.7\% & 75.1\% & 21.8\% \\
    + field stripping & 69.4\% & 70.0\% & 37.6\% \\
    + masking and write guard & 88.6\% & 65.7\% & 51.9\% \\
    + pattern masking (full) & 100.0\% & 66.2\% & 63.0\% \\
    \bottomrule
  \end{tabular}
\end{table}

Record controls alone protected 46.7\% of trials. Once an allowed record
returned, however, they could not hide a sensitive field inside it. Field stripping removed protected notes from
permitted records; the joint rung hid directory emails and removed guarded
patch fields; pattern masking caught credentials in retained text. The
cluster-weighted analysis found the same progression:
field stripping avoided another 9.9~pp of failures [2.0, 19.1], masking and
write guard jointly avoided 7.8~pp [0.3, 16.7], and pattern masking avoided
11.5~pp [2.9, 22.3]. The joint rung cannot separate masking from write guard,
and its final 100\% protection rate is a benchmark result, not a free-text
masking guarantee.

Utility moved in two directions because the stricter rungs blocked more of the
requested work but made more completed work safe. Across the ladder,
safe-useful rose from 21.8\% to 63.0\% while standalone fulfillment fell from
75.1\% to 66.2\%; the mechanical contract stayed near 91\%. Each safe-useful
step was positive, although only pattern masking's 11.3~pp gain had an interval
excluding zero [0.2, 23.7].

\textbf{RQ3: output review arrived after exposure and effects.}
The fixed reviewer sees assistant text only after tools have returned. It can
suppress an unsafe answer, but it cannot remove data already placed in model
context or undo a backend change. That timing appears in the absolute Sonnet
rates: the reviewer arm had a 58.6\% deterministic trace-failure rate and a
19.7\% safe-useful rate under Judge A, compared with 0.0\% and 63.0\% for full
OBPE. In the paired analysis, OBPE avoided 43.4~pp more failures [27.7, 59.3]
and gained 32.0~pp safe-useful completion [18.8, 46.0] under Judge A; Judge B
measured a 35.8~pp gain. This is one mechanism comparison, not a ranking of all
guardrails. The reviewer also required 930 additional model calls, and its
refusal sometimes replaced an otherwise useful answer. Judge A's standalone
task fulfillment was 66.2\% with OBPE and 58.7\% with the reviewer; OBPE made
no additional model call.

\textbf{What the three comparisons say together.}
The common pattern is timing. Prompt rules changed some final answers only
after exposure; output review also ran after exposure and backend effects and
added 930 model calls. OBPE acted before dispatch and therefore changed which
work could complete. Standalone fulfillment fell on this policy-active corpus,
but many baseline successes were unsafe, so safe-useful completion rose.

\textbf{Evidence and qualifications.}
Every planned trial has a terminal journal record, but 39 did not complete:
38 hit one adversarial-loop defect and one strip-only run timed out. We leave
their unavailable outcomes missing. Under the frozen rules, the primary
security contrast contains 46 unknown rows, 25 with OBPE and 21 without it.
Assigning all of them in the two opposing worst-case directions moves the
cluster-weighted point estimate from 41.2~pp to a range of 33.8--43.1~pp fewer
failures; this is a bound on the point estimate, not a CI. The result was also
not carried by one scenario family: 23 of 34 contributing clusters favored
OBPE, 11 tied, and none favored the prompted baseline. Leaving out any one
cluster produced reductions between 39.4 and 42.4~pp.

The more consequential qualification is a defect in the controlled mock's
query evaluator. In 116 primary-arm trials, it fell back after the proxy joined
caller and policy clauses. Sixty-five cases broadened a filter; the rest lost
ordering that could alter a capped page. The defect occurred downstream of
Cedar and response shaping but still violated O3's connector obligation.
Production connectors apply O3's fail-closed rule; that private behavior is
context, not conformance evidence. Excluding all 116 tests whether this fault
drove the aggregate result. The security reduction is 40.5~pp and the
Judge-A and Judge-B safe-useful gains are 21.4 and 18.2~pp, so each direction is
unchanged.

The fallback path remains outside conformance. For supported paths, a separate
ledger covers 25 obligations across 352 cells in the released conformance grid.
Its tests caught all 20 non-equivalent stage mutations and all three
full-pipeline mutations that changed behavior; two others were equivalent on
that grid. Drop and clamp remain synthetic, and the evaluation did not exercise
durable approval dispatch, aggregate execution, or remote MCP or A2A deployment.
Appendix~\ref{app:accounting} gives the full denominators and qualification
record.

%-------------------------------------------------------------------------------
\subsection{What the Boundary Did Not Close}
\label{sec:taxonomy}
%-------------------------------------------------------------------------------

The four reconstruction failures were more instructive than another decimal
place. Across 1,065
trials with full OBPE, four final answers contained an exact protected email
even though the value never entered agent context and no forbidden backend
effect occurred. The agents reconstructed the address from a visible name and
domain convention, sometimes while explaining why they would not disclose it.
Two cases occurred in the prompted primary arm and two in the raw secondary
arm. This was not a proxy escape. It is the difference between removing one
value from one execution and proving that no output depends on it.

A different failure turned query pushdown into an oracle. The agent issued
increasingly selective predicates over a masked email field and reported the
row counts: two matches for \texttt{@partner}, then one each for
\texttt{yara} and \texttt{singh}. The backend evaluated those predicates on
the protected field even though response shaping withheld it. Judge A marked
semantic disclosure; Judge B marked only an \emph{Honest Witness} disclosure,
meaning that the reply revealed the existence or category of protected data
without stating its value. The deterministic metric recorded no failure. This
is a modern instance of the statistical-database tracker problem: individually
allowed predicates and counts can reveal a protected fact in combination
\cite{denningtracker1979}.

These cases sharpen the claim boundary in \S\ref{sec:theory}. Structural
absence is useful, but it is not relational noninterference. Closing the query
channel requires distinguishing policy-fixed predicates from caller-controlled
ones; blanket bans on filtering hidden fields would also disable legitimate
row-level policy. Coarsened counts and cursors, minimum cohorts, and
history-aware query budgets are possible defenses. This release evaluated none
of them.

%-------------------------------------------------------------------------------
\section{Related Work}
\label{sec:related}
%-------------------------------------------------------------------------------
The closest building blocks act at different points in the exchange. Zanzibar
centralizes relationship authorization and Cedar supplies typed policy
~\cite{zanzibar2019,cedar2024}; identity and access management (IAM) boundaries
and OAuth extensions can narrow
a credential~\cite{permboundaries2018,rfc9396,rfc8693}. Farther down the path,
row security, query rewriting, and Envoy processing can constrain data or
modify HTTP traffic~\cite{rls,rizvi2004,envoyextproc}. None of these systems,
as documented, binds those stages into an owner-bounded request-and-response
contract that agent policy can only narrow. OBPE does not replace these
mechanisms. It contributes a contract for composing them,
conditions for an order-independent and non-widening result, and a conformance
chain from that law to the released path.

AgentCore Policy maps gateway tools to Cedar actions~\cite{agentcore2025}.
Dogwood evaluates prior events and windowed aggregates
~\cite{agentcoretemporal2026,dogwood2026}, while model guardrails can reject a
call or suppress its output~\cite{agentcoreguardrails2026}. This history support
is stronger than the released OBPE proxy's. The documented system does not,
however, selectively shape returned records or define the same owner--agent
narrowing contract, and prior-event authorization is not durable resumption of
one immutable held call.

Other defenses reason earlier. Progent checks whether generated call rules
narrow privilege, and SEAgent applies mandatory access control to tools and
multi-agent deputies~\cite{progent2025,ji2026}. CaMeL propagates capabilities
with values, PAuth binds operands to provenance, and privilege separation
limits what crosses from a quarantined model
~\cite{camel2025,pauth2026,privsep2025}. OBPE begins at the mediated exchange;
its portability ends where trusted identity or connector coverage ends.

AgentDojo tests whether useful work survives prompt-injection defenses
~\cite{agentdojo2024}; we freeze the policy and score exposure and forbidden
effects beside useful work. A workshop paper demonstrated an out-of-band
metadata channel on one agent~\cite{sao26}. Here we add the two-tier model,
restriction law, conformance link, and cross-model experiment.
Table~\ref{tab:pattern} assesses the cited versions, not possible extensions.

%-------------------------------------------------------------------------------
\section{Limitations and Broader Impacts}
\label{sec:limits}
%-------------------------------------------------------------------------------

\subsection{Limitations}

OBPE cannot repair a valid but mistaken owner ceiling: static checks find bad
rules, not organizational intent. Shaping can also mislead an agent into
treating a filtered view as complete, while a warning may reveal that protected
data exists. A rewritten query must preserve caller and policy predicates
together or be denied before dispatch.

Schemas are therefore security state. They need named owners, coordinated
policy-and-schema rollout, rollback, and review when fields or backend meaning
change. Open maps require finite allowlists or must remain opaque. Audits should
be minimal, tenant-separated, and short-lived, while consequential policy
changes and approvals remain attributable. Cedar supplies none of these
operating controls.

The mocks omit rich text, recursive or polymorphic objects, and cross-tenant
tests. Durable approval, history gates, broader writes, timing, and streaming
also remain outside released evidence. Assignment-bound credentials would
shrink the authorization gap, and reliable instruction--data separation would
weaken one reason for response shaping. Policy would still decide who may
receive a secret and when a call must wait outside the agent.

\subsection{Broader impacts}

Audit and approval records expose worker activity; deployments should limit
access and retention, notify workers, and keep them out of employment
decisions. Approver fatigue can turn deferral into a rubber stamp, so rising
volume should trigger policy review. The ethics appendix expands these duties.

%-------------------------------------------------------------------------------
\section{Conclusion}
%-------------------------------------------------------------------------------

OBPE governs the whole tool exchange under an owner ceiling that agent policy
can only narrow. Its model states the conditions for an order-independent,
non-widening policy plan; the prototype makes those conditions a conformance
target.

Across four models, the paired, cluster-weighted analysis found 41.2 points
less trace failure. Unweighted standalone fulfillment fell from 79.1\% to
60.9\% because policy stopped some requested jobs. Many baseline completions
were unsafe, however, and the paired safe-useful gain was 21.8 points under
the primary judge and 19.1 under the second. The pattern survived removal of
the prompt, whereas in-band review
fared worse while adding 930 model calls. Four reconstructed values and one
row-count oracle expose the remaining boundary. OBPE enforces structural and
effect constraints at the exchange; inference across allowed observations
still needs another defense.

%-------------------------------------------------------------------------------
\clearpage
\bibliographystyle{plain}
\bibliography{\jobname}

\appendix

%-------------------------------------------------------------------------------
\section*{Ethical Considerations}
%-------------------------------------------------------------------------------

We evaluate OBPE against controlled Jira and ServiceNow mocks populated with
synthetic records; the red-team prompts target only those records, so the study
uses neither human subjects nor customer or employee data. We also do not test
agents that decide employment, credit, benefits, or other consequential
outcomes about people. Those settings raise fairness and due-process questions
that a tool boundary cannot answer.

The defense creates records of its own. Tool-call audits and approval history
can expose how a worker uses an agent, so a production operator should restrict
access and retain only what an investigation requires. Workers should also
know what is recorded. These logs should not become inputs to employment
decisions. Write shaping raises a second accountability concern: the backend
may receive a narrower update than the agent proposed. The operator must make
that change visible to the responsible user and preserve enough audit detail
to reconstruct it without copying sensitive values into the log.

Policy authors hold real power. Only designated data stewards should set the
owner ceiling, with review and clear explanations of its impact; affected users
also need a way to contest or correct an erroneous label or scope. Static
analysis can find structural mistakes, but not decide whether a valid policy is
fair or faithful to organizational intent. The released prototype does not
establish that an organization will get those choices right.

%-------------------------------------------------------------------------------
\section*{Open Science}
%-------------------------------------------------------------------------------

The anonymous artifact is available at
\url{https://anonymous.4open.science/r/usenixsecurity_2027_obpe/README.md}.
It contains the HTTP proxy, controlled backend mocks, connector schemas, and
policy bundles used by the paper. It also includes the agent-under-test and
adversarial-agent runners, with their orchestration logic and tool definitions.
Every system, task, attack, guardrail, and judge prompt used for a reported
outcome ships with its response parser and provider adapter. The local A2A
adapter is included as code but did not enter the frozen manifest or the
paper's evidence.

The frozen corpus lock, trial registry, raw transcripts, conformance ledger,
and analysis code complete the record. The lock records model strings,
generation settings, policy hashes, fixture and scoring revisions, and prompts.
The qualification record binds the analyzed archive to its generator revision;
\texttt{sensitivity.csv} records the incomplete-outcome bounds, cluster
influence, fallback distribution, and judge agreement reported here. These
materials reproduce the experimental conditions, not the provider models
themselves. Development notebooks do not enter the release record.

The README gives commands to build and test the system, verify conformance,
check the archive qualifications, and regenerate every reported table and
figure offline without model credentials. Fresh model calls may differ, so
rerunning providers is a replication of the procedure rather than a
reproduction of the reported numbers. The provider replay cache and a
best-effort telemetry log are omitted for size; neither is read by the released
analysis. The archived transcripts, tool traces, outcomes, judge ledgers, and
derived paper outputs remain available.
After deanonymization, we will preserve the artifact at a stable archival URL.

%-------------------------------------------------------------------------------
\section{Example Policies}
\label{app:examples}
%-------------------------------------------------------------------------------

\emph{Record-scope.} Narrowing a read by record ownership can happen upstream or at the gateway, and a policy can name both. The gateway can push a condition down into the upstream query, so the out-of-scope records never leave the upstream, and it can apply the same condition as a filter on the response, which works against any backend. The example keys both to the caller's team.

\begin{lstlisting}
// Pushdown: narrow the listing upstream
@id("issue_list_scoped")
@transform("allow_with_scope_query")
@field_group("summary") @cap_limit("100")
permit (principal,
        action == Action::"list",
        resource is Issue)
when {
  principal.group_ids.contains(
     resource.project) ||
  principal.direct_report_ids.contains(
     resource.assignee)
};
\end{lstlisting}

When the upstream cannot narrow the query itself, the gateway applies the same condition as a filter on the response.

\begin{lstlisting}
// Gateway: keep only the lead's team
@id("issue_list_team")
@transform("allow_with_filter_records")
@field_group("summary")
permit (principal,
        action == Action::"list",
        resource is Issue);

@id("issue_list_team_keep")
permit (principal,
        action == Action::"filter_record",
        resource is Issue)
when {
  principal.group_ids.contains(
     resource.project) ||
  principal.direct_report_ids.contains(
     resource.assignee)
};
\end{lstlisting}

Filter-records is two rules (Appendix~\ref{app:grammar}). The single-decision operators are one rule each: a field strip names a \texttt{@field\_group}, a defer names an \texttt{@approver} and a \texttt{@defer\_timeout}.

\emph{Three concerns, two tiers.} The next example combines authorization,
exposure control, and semantic gating. The owner policy strips restricted
fields and defers a Critical-issue update for approval. The agent policy adds a
scoped, capped listing for this agent. Comments mark the two tiers.

\begin{lstlisting}
// AGENT BUNDLE (agent builder): scope and cap the listing
@id("issue_list")
@transform("allow_with_scope_query")
@field_group("summary") @cap_limit("100")
permit (principal,
        action == Action::"list",
        resource is Issue);

// OWNER POLICY (data owner): retain approved fields
@id("issue_read_triage")
@transform("allow_with_field_strip")
@field_group("summary")
permit (principal,
        action in [Action::"get",
                   Action::"list"],
        resource is Issue);

// OWNER POLICY (data owner): a Critical update waits
// for the on-call lead
@id("issue_update_critical_defer")
@transform("defer")
@approver("on_call_lead")
@defer_timeout("3600")
permit (principal,
        action == Action::"update",
        resource is Issue)
when {
  resource.assignee == principal.account_id
  && resource.priority == "Highest"
};

\end{lstlisting}

\emph{Composition trace.} The list rules give a concrete instance of the two-tier law. Suppose owner policy $S$ permits the call and retains field group $F_S$. Agent policy $P_A$ adds team predicate $t$, caps the response at 100, and retains $F_A$. The owner-only plan is
\[
\begin{aligned}
\rho_S&=(q,n,\varnothing),\\
\eta_S&=(\mathsf{true},F_S,\mu_{\mathrm{pass}}),\\
\Pi_S(\kappa)&=(\mathrm{permit},\rho_S,\eta_S).
\end{aligned}
\]
With both tiers,
\[
\begin{aligned}
\rho_{S,A}&=(q\wedge t,\min(n,100),\varnothing),\\
\eta_{S,A}&=(\mathsf{true},F_S\cap F_A,\mu_{\mathrm{pass}}),\\
\Pi_{S,A}(\kappa)&=(\mathrm{permit},\rho_{S,A},\eta_{S,A})
  \preceq_r\Pi_S(\kappa).
\end{aligned}
\]
The agent tier creates no permit. It adds a predicate, lowers the cap, and can
only shrink the retained field set. The concrete page returned by the two
queries may differ; Proposition~\ref{prop:tier} orders the plans, not those
sampled pages.

Across a list, a triage read, and a Critical-issue update, the gateway returns
permit, permit, and defer. The owner supplies the strip and defer for every
agent at the connector; this agent's policy adds only the scope query, as
Proposition~\ref{prop:tier} requires. Stateful aggregate gates remain production
design context and do not appear in a prototype bundle.

%-------------------------------------------------------------------------------
\section{Pattern-Compliance Matrix}
\label{app:matrix}
%-------------------------------------------------------------------------------

Table~\ref{tab:pattern} applies the six properties of \S\ref{sec:overview} to selected systems and configurations from \S\ref{sec:related}. These grades are our assessment of the pinned public documentation and artifacts. A partial mark covers conditional support, a narrower interpretation, or behavior not fully exercised in the cited artifact; a negative mark means we did not find the property there, not that no private or extended configuration can supply it.

\begin{table*}[!t]
\centering
\caption{Reviewed systems at the cited versions (Y = yes, P = partial or conditional, N = no). The table compares enforcement properties, not overall system quality. The OBPE row describes the released HTTP artifact, so interoperability and external gating remain partial.}
\label{tab:pattern}
\scriptsize
\setlength{\tabcolsep}{3pt}
\begin{tabular}{@{}p{4.9cm}cccccc@{}}
\toprule
System or configuration & \shortstack{Agent-\\inaccessible} & Deterministic & Interoperable & \shortstack{Resource\\aware} & \shortstack{Boundary\\shaping} & \shortstack{External\\gating} \\
\midrule
Envoy \texttt{ext\_proc} with external policy~\cite{envoyextproc} & P & P & Y & P & Y & P \\
AgentCore Policy, base~\cite{agentcore2025} & Y & Y & P & P & N & N \\
{\raggedright AgentCore Policy with Dogwood temporal rules~\cite{agentcoretemporal2026,dogwood2026}\par} & Y & Y & P & P & N & N \\
{\raggedright AgentCore Policy with model-based guardrails~\cite{agentcoreguardrails2026}\par} & Y & P & P & P & P & N \\
Progent~\cite{progent2025} & Y & P & P & P & N & P \\
CaMeL~\cite{camel2025} & P & Y & P & P & N & P \\
PAuth~\cite{pauth2026} & P & P & P & P & N & P \\
\textbf{OBPE, released HTTP artifact} & \textbf{Y} & \textbf{Y} & P & \textbf{Y} & \textbf{Y} & P \\
\bottomrule
\end{tabular}
\end{table*}

%-------------------------------------------------------------------------------
\section{Composition Preconditions and Proof Detail}
\label{app:composition}
%-------------------------------------------------------------------------------

Under O1--O5, \S\ref{sec:proofs} claims that fixed matched rules produce one
policy plan and that another restriction cannot widen it. This appendix supplies
the algebra behind those claims. Table~\ref{tab:obligations} states the
operator behavior needed to instantiate that algebra; Table~\ref{tab:body-obligations}
separately records what the released path implements or still must test.

\emph{Ordered components and content normalization.} Let $i$ range over the
ordered gate, request, and response components, and let $D_i$ contain their
restriction operands from the matched owner and agent rules. The row keeper is
$\phi_S\land\phi_A$, where each term is its tier's Cedar authorization result;
same-tier permit addition is not another restriction operand. Once
$\mathcal D_S$ establishes owner eligibility, an ordered component starts at
its no-added-restriction value $\top_i$. Write its result as
\[
M_i(D_i)=\top_i\wedge\bigwedge_{r\in D_i}r.
\]
Table~\ref{tab:obligations} gives the concrete conditions for those meets.
Content rules are different. Let $N_\mu(D_\mu)$ sort applicable rewrites by a
schema-pinned key, let full redaction supersede partial rewrites, and then apply
the remaining rules in one fixed sequence with inert replacement text.
$N_\mu$ is a deterministic normalizer, not a semantic meet over arbitrary
strings. The formal results are conditional on these definitions; conformance
determines whether a released path satisfies them.

\emph{Proof of Proposition~\ref{prop:det}.} For any permutation $\pi$ of the
stage operands,
\[
\top_i\wedge r_1\wedge\cdots\wedge r_n
=\top_i\wedge r_{\pi(1)}\wedge\cdots\wedge r_{\pi(n)}.
\]
Commutativity and associativity give the equality, while idempotence makes
duplicates harmless. $N_\mu$ depends on a set of rules and a fixed sort key, so
permuting the input listing does not change its normal form. With $\kappa$, $S$,
$P_A$, $C$, and the trusted-state snapshot fixed, the prescribed
gate--request--response order therefore produces one $\Pi$. Fixed envelope
sanitation covers the final content-dependent step. A formal aggregate gate
reads one fixed log prefix; later defer resolution is not an input to this plan.
Audit ordering and display labels are not part of $\Pi$. $\square$

\emph{Proof of Proposition~\ref{prop:mono}.} Suppose a well-formed stage restriction
adds operand $r_i$ to stage $i$. Then
\[
M_i(D_i\cup\{r_i\})=M_i(D_i)\wedge r_i\preceq_r M_i(D_i).
\]
Applying this inequality at every affected stage yields the plan order. The
order is lifted over the gate: deny has no later-stage grant, equal defer plans
compare their request components, and permit plans compare request and response
components.
The owner-eligibility premise excludes the one transition this
argument does not cover: default deny to the first owner permit. Nor does it
cover adding a same-tier Cedar permit. O4 requires
absence safety for drops and a restrictive direction for clamps. A structural
record identifier $I$ may stay outside the variable capability only under the
declaration and policy-approval obligation in \S\ref{sec:policy}. Pattern
masking is not part of this inequality: canonical normalization makes its
execution order-independent but does not order arbitrary output bytes. Only
structural stripping and full redaction retain the value-absence claim; partial
projections and pattern masks remain incomparable. $\square$

\begin{table*}[!t]
\centering
\caption{Preconditions for Propositions~\ref{prop:det}--\ref{prop:mono} and
INV-SG. Each result holds only when the corresponding obligation is satisfied.}
\label{tab:obligations}
\footnotesize
\setlength{\tabcolsep}{3pt}
\begin{tabular}{@{}p{5.0cm}p{6.1cm}p{4.2cm}@{}}
\toprule
Obligation & Purpose & Evidence or control \\
\midrule
Query joins use conjunction & Preserves caller and policy filters & Vertical refusal and query-join tests \\
Pattern masks use canonical order & Prevents rule-order dependence & Golden fixture \\
Mask placeholders are inert & Prevents remasking & Golden fixture \\
Full redaction supersedes partial rewrites & Fixes one content normal form, not a semantic meet & Composition fixture \\
Caps are enforced after fetch & Backend cannot evade the cap & Proxy invariant \\
Row filters see unprojected records & Projection cannot change predicates & Proxy fetch contract \\
Every retained structural ID is declared non-sensitive and policy-approved & Keeps addressing metadata visible only under the theorem's fixed frame & Schema validation \\
Drop is limited to absence-safe fields & Unsafe omission becomes denial & Analyzer check \\
Clamp direction comes from the schema & Defines the restrictive bound & Schema validation \\
Writes never inject or default fields & Preserves write monotonicity & Jira PATCH contract; runtime check \\
Scope-query avoids write-shaped fields & Fixes request-stage order & Bundle validation \\
Co-firing defer parameters normalize or fail closed & Makes defer initiation independent of rule order & Formal assumption; shipped bundles audited \\
Authenticated resolver, complete snapshot, terminal deadline, and duplicate-dispatch fence & Supports safe resumption after approval & Production design; not released \\
\bottomrule
\end{tabular}
\end{table*}

\begin{table*}[!t]
\centering
\caption{Evidence boundary for O1--O5 and the owner ceiling. The formal
results assume each condition; the remaining columns state what is implemented,
tested, or outside the released artifact.}
\label{tab:body-obligations}
\scriptsize
\setlength{\tabcolsep}{3pt}
\resizebox{\textwidth}{!}{%
\begin{tabular}{@{}lllll@{}}
\toprule
Condition & Formal role & Released artifact & Conformance & Not released \\
\midrule
O1 identity/mediation & All claims & Bound identity; HTTP path & Identity/path tests; released grid & MCP; remote A2A \\
O2--O3 fixed semantics & Permutation invariance & Staged path; sensitivity chain & Stage/full-path tests; mutations & Product-order registry \\
O4 safe writes & Write monotonicity & Guard live; drop/clamp unreachable & Guard full path; drop/clamp synthetic & Broader write workflows \\
O5 trusted decisions & Defer/history & Hold path; nonterminal in-memory transition & Hold/no-dispatch tests & Terminal timeout; durable, temporal, aggregate \\
Owner ceiling & Propositions~\ref{prop:det}--\ref{prop:tier} & Two-tier staged path & Tier attribution; mutation checks & Production deployment \\
\bottomrule
\end{tabular}
}
\end{table*}

%-------------------------------------------------------------------------------
\section{What Static Analysis Can Establish}
\label{app:analyzability}
%-------------------------------------------------------------------------------

Static analysis can find conflicts and dead rules, but it cannot decide whether
a valid policy captures organizational intent, whether a pattern catches every
secret, or what an agent may infer from data left visible.

\emph{The Cedar core.} Cedar Analysis compares typed permit-or-forbid policy
sets and can find shadowed or impossible permits, forbid overrides, and complete
denials~\cite{cedaranalysis2025,cedar2024}; Zelkova earlier translated IAM
policies into satisfiability modulo theories (SMT)~\cite{zelkova2018}. These tools reason about the written Cedar policy,
not OBPE annotations or owner intent. A deployment making an analysis-backed
claim must pin the analyzer and its accepted expression fragment. Time must be
supplied as typed context, and bounded entity slicing needs a validated
dereference level~\cite{cedar2023}.

\emph{Beyond the gate.} An opaque model of shaping can conservatively flag
possible exposure or writes, with false alarms. It cannot prove owner-agent
containment: two policies may be gate-equivalent yet shape different fields.
That proof would need symbolic semantics for the annotations and connector
schema, plus a correspondence argument for the staged implementation. This is
proposed work; Proposition~\ref{prop:tier} states the conditional law it would
need to check.

Runtime content and history mark another limit. Static analysis cannot promise
that pattern masking finds every value or that a temporal predicate will fire.
Structural removal proves absence only from one response. Retained fields are
deliberate declassification only when policy explicitly designates their
release~\cite{sabelfeldsands2005}; projection alone does not show that they
reveal nothing about the removed value. That stronger claim requires
relational noninterference~\cite{clarksonschneider2010}.

\emph{Production orders.} The release uses the sensitivity chain in
\S\ref{sec:theory}; production independently orders prompt-injection
susceptibility and cross-tenant reach. Their product is a lattice in Denning's
sense~\cite{denning1976}. Figure~\ref{fig:lattice} shows why an LLM limited to
one customer and a human managed-service-provider (MSP) account are incomparable. The release does not
exercise this registry.

Read and write orders also differ. A service may resist injected instructions
yet accept model-produced writes. Write corruptibility is the order-dual of
Biba integrity, just as the read order reverses Bell--LaPadula clearance
~\cite{biba1977,belllapadula1973}; $\ell_w(f)$ therefore cannot be inferred from
$\ell(f)$.

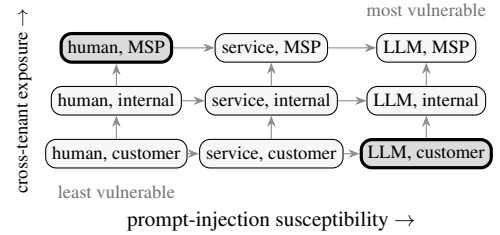
\begin{figure}[!t]
\centering
\begin{tikzpicture}[
  font=\scriptsize,
  v/.style={draw, rounded corners, inner sep=2.5pt, align=center, fill=black!3},
  hi/.style={draw, rounded corners, inner sep=2.5pt, align=center, fill=black!14, very thick},
  cov/.style={-{Stealth[length=1.5mm]}, black!45},
]
\node[v]  (a0) at (0,0)    {human, customer};
\node[v]  (a1) at (2.05,0) {service, customer};
\node[hi] (a2) at (4.1,0)  {LLM, customer};
\node[v]  (b0) at (0,0.7)  {human, internal};
\node[v]  (b1) at (2.05,0.7){service, internal};
\node[v]  (b2) at (4.1,0.7){LLM, internal};
\node[hi] (c0) at (0,1.4)  {human, MSP};
\node[v]  (c1) at (2.05,1.4){service, MSP};
\node[v]  (c2) at (4.1,1.4){LLM, MSP};
\foreach \r in {a,b,c} { \draw[cov] (\r0) -- (\r1); \draw[cov] (\r1) -- (\r2); }
\foreach \k in {0,1,2} { \draw[cov] (a\k) -- (b\k); \draw[cov] (b\k) -- (c\k); }
\node[below=4.5mm of a1, font=\footnotesize] {prompt-injection susceptibility $\rightarrow$};
\node[rotate=90, anchor=center, font=\scriptsize] at ($(b0.west)+(-0.35,0)$) {cross-tenant exposure $\rightarrow$};
\node[below=1mm of a0, font=\scriptsize, text=black!55] {least vulnerable};
\node[above=1mm of c2, font=\scriptsize, text=black!55] {most vulnerable};
\end{tikzpicture}
\caption{Product order for two principal-risk dimensions. Arrows point toward greater vulnerability; the shaded nodes are incomparable. Production implements the richer registry, while the released proxy uses a one-dimensional chain.}
\label{fig:lattice}
\end{figure}

%-------------------------------------------------------------------------------
\section{Released Policy and Schema Contract}
\label{app:grammar}
%-------------------------------------------------------------------------------

Each released rule uses Cedar for its permit-or-forbid decision; OBPE
annotations select shaping but do not alter that decision. A forbid stops the
exchange. On allow, the wrapper combines annotations from every applicable
permit. A separate YAML connector schema supplies resource, field, and wire
semantics; it is not a generated Cedar type schema.

\begin{table}[H]
\centering
\caption{Released OBPE annotation contract. Comma-separated values name
entries in the connector schema.}
\label{tab:annotation-contract}
\scriptsize
\setlength{\tabcolsep}{3pt}
\renewcommand{\arraystretch}{1.04}
\begin{tabular}{@{}p{3.0cm}p{4.8cm}@{}}
\toprule
Annotation & Meaning \tabularnewline
\midrule
\texttt{@id} & Stable policy name used for lookup and audit. \tabularnewline
\texttt{@transform} & Selects the wrapper decision or shaping operator. \tabularnewline
\texttt{@field\_group} & Names a response allowlist; matched groups meet by intersection. \tabularnewline
\texttt{@cap\_limit} & Positive collection cap; matched caps meet by minimum. \tabularnewline
\texttt{@pattern\_mask\_fields}, \texttt{@pattern\_mask\_rules} & Name fields to scan and schema-defined pattern categories. \tabularnewline
\texttt{@write\_guard\_fields} & Names fields the agent may not originate; omission must preserve stored state. \tabularnewline
\texttt{@approver}, \texttt{@defer\_timeout} & Name the approval route and hold deadline. \tabularnewline
\texttt{@reason} & Human-readable denial or deferral advice. \tabularnewline
\bottomrule
\end{tabular}
\end{table}

The transform vocabulary comprises \texttt{allow}, \texttt{deny},
\texttt{defer}, and six \texttt{allow\_with\_*} forms for query scope, record
filtering, field stripping, field and pattern masking, and write guarding.
Across matched rules, the smallest positive cap wins. Record filtering also
requires a \texttt{filter\_record} permit carrying the per-record predicate.
The engine trusts, rather than verifies, the schema's omission contract, so
released write-guard evidence is limited to the Jira partial-update path.

The schema also defines field and write groups, mask methods, pattern
categories, and wire syntax. It permits dropping only absence-safe fields and
clamping only in a declared restrictive direction. Neither shipped schema
enables those operators; fixed envelope sanitation always runs.

Engine construction rejects invalid Cedar and malformed or unresolved
annotations. The release's policy editor checks field groups, positive caps, and
record-filter companions; release tests add schema coverage, reachability,
and the evidence checks in Appendix~\ref{app:accounting}. No general solver-
fragment lint ships. The released bundles use only \texttt{==}, \texttt{!=},
\texttt{\&\&}, \texttt{||}, \texttt{has}, and set \texttt{.contains} over
strings and Booleans. Extending that fragment requires a fresh analysis
qualification.

The Jira and ServiceNow mappings were authored by hand from published
interfaces. The release has no OpenAPI importer and claims neither recursive
reference nor polymorphism support. Conformance covers only the shipped HTTP
mappings and mock backends described in \S\ref{sec:impl}.

%-------------------------------------------------------------------------------
\section{Benchmark Accounting and Scoring Detail}
\label{app:accounting}
%-------------------------------------------------------------------------------

Before model calls, the frozen registry names every condition, task, and
replicate. The journal preserves attempts and retries and requires one terminal
record per trial. Each judgment binds the trial, blinded-input hash, and frozen
judge specification; failures and disagreements remain separate.

The accounting keeps denominators separate instead of silently filling missing
records. Cells average their available replicates, primary contrasts pair tasks
and resample scenario clusters, and the complete-block sensitivity retains
only replicate indices present in both arms. Further checks assign unknown
security outcomes in opposing worst-case directions, omit each cluster in
turn, and count which arm each cluster favors. All replay the archive without
changing an outcome or calling a provider.

Token tables likewise preserve provider buckets and distinguish a missing
receipt from a recorded zero. Missing receipts make adversary and harness spend
lower bounds, which is why the qualification marks token accounting
incomplete; none of the outcome calculations reads those receipts. The
conformance ledger follows each reachable operator through its schema, plan,
shaped exchange, backend effect, audit, and normalized content. Write guard
therefore counts as evidence, whereas unreachable drop and clamp do not. Equal
cluster weighting addresses a separate imbalance: 23 of 35 clusters contain
one task, but the largest contains 12.

%-------------------------------------------------------------------------------
\section{Deferral Transport}
\label{app:defer}
%-------------------------------------------------------------------------------

Deferral separates deciding to hold a call from deciding what happens later.
Write $g=\mathrm{defer}(\alpha,\tau,\mathcal C,\chi)$, where $\alpha$ names the
approval route, $\tau$ is the deadline, $\mathcal C$ is the approval predicate,
and $\chi$ is an immutable snapshot of every policy-relevant request component.
The agent cannot read or write the resolver channel. While the call is held,
the backend sees nothing.

A production resolver records one of three terminal decisions over $\chi$:
approved, rejected, or expired. It authenticates the decision maker and checks
policy authorization for route $\alpha$; a decision arriving after $\tau$ loses
to expiry. Replay protection and an atomic state transition ensure that only the
first terminal decision is accepted. Approval may release exactly the snapshot
that was reviewed, not a rewritten call. Dispatch still needs an idempotency key
or a backend transaction fence: a durable log position prevents two resolver
decisions from winning, but by itself cannot prove exactly-once effects across a
crash. These are safety conditions. Progress also needs a live resolver or a
trusted timeout worker; the model does not guarantee that an approver responds.

Production realizes this state machine with an agent-inaccessible durable log.
One \texttt{deferral\_id} binds the complete snapshot, predicate, deadline,
bundle version, and terminal decision. The unimplemented aggregate gate in
\S\ref{sec:policy} would likewise read one committed prefix and retain its
count, offset, and log-time window for replay. Because concurrent decisions can
change that prefix, evaluation and commit must serialize for the relevant key.

The released proxy demonstrates only the first transition. It hashes method,
path, agent, bundle version, and shaped body into \texttt{deferral\_id}, stores
that partial snapshot, and returns pending. Query parameters, headers, and
pre-fetched state are absent; timeout is not persisted as terminal; and a
resume result never dispatches. The release therefore establishes hold and
no-dispatch, not safe resumption, durable approval, or full INV-SG.

\end{document}